\documentclass{article}

\usepackage[final]{neurips_2023}

\usepackage[utf8]{inputenc} 
\usepackage[T1]{fontenc}    
\usepackage{hyperref}       
\usepackage{url}            
\usepackage{booktabs}       
\usepackage{amsfonts,mathrsfs}       
\usepackage{nicefrac}       
\usepackage{microtype}      
\usepackage{xcolor}         

\usepackage{amsmath}
\usepackage{bm,color}
\usepackage{amssymb}
\usepackage{amsthm}
\usepackage{multicol}
\usepackage{multirow}
\usepackage{mathtools}
\usepackage{mathrsfs} 
\usepackage[capitalise, nameinlink]{cleveref}
\usepackage[normalem]{ulem}
\usepackage{algorithmic}
\usepackage{subfigure}
\usepackage{graphicx}
\usepackage{pifont}
\usepackage{threeparttable}
\usepackage{graphicx}
\usepackage{amsbsy}
\usepackage[normalem]{ulem}
\usepackage{enumitem}
\usepackage{natbib}
\usepackage[pdf]{graphviz}
\usepackage{algorithm}
\usepackage{enumitem}

\usepackage{crossreftools}
\newtheorem{theorem}{Theorem}
\newtheorem{definition}{Definition}
\newtheorem{lemma}{Lemma}

\newtheorem{assumption}{Assumption}
\theoremstyle{plain}

\theoremstyle{definition}
\theoremstyle{remark}

\usepackage{selectp}

\title{Sample Complexity Bounds for Score-Matching: Causal Discovery and Generative Modeling}

\author{%
Zhenyu Zhu$^\dagger$, $\quad$ Francesco Locatello$^{\ddagger}$$^\ast$, $\quad$ Volkan Cevher$^\dagger$\thanks{Share the senior authorship}\\
$\dagger$ \'{E}cole Polytechnique F\'{e}d\'{e}rale de Lausanne $\quad^{\ddagger}$ Institute of Science and Technology Austria\\
  {$\dagger$\texttt{\{zhenyu.zhu, volkan.cevher\}@epfl.ch}
  $\quad^{\ddagger}$\texttt{francesco.locatello@ista.ac.at}}
}

\begin{document}

\maketitle

\begin{abstract}
This paper provides statistical sample complexity bounds for score-matching and its applications in causal discovery. We demonstrate that accurate estimation of the score function is achievable by training a standard deep ReLU neural network using stochastic gradient descent. We establish bounds on the error rate of recovering causal relationships using the score-matching-based causal discovery method of \citet{rolland2022score}, assuming a sufficiently good estimation of the score function. Finally, we analyze the upper bound of score-matching estimation within the score-based generative modeling, which has been applied for causal discovery but is also of independent interest within the domain of generative models.
\end{abstract}

\section{Introduction}
\label{sec:intro}

Score matching~\cite{hyvarinen2005estimation}, an alternative to the maximum likelihood principle for unnormalized probability density models with intractable partition functions, has recently emerged as a new state-of-the-art approach that leverages machine learning for scalable and accurate causal discovery from observational data~\cite{rolland2022score}. However, the theoretical analysis and guarantees in the finite sample regime are underexplored for causal discovery even beyond score-matching approaches.

\textbf{Contributions:} \ In this work, we give the first sample complexity error bounds for score-matching using deep ReLU neural networks. With this, we obtain the first upper bound on the error rate of the method proposed by \citet{rolland2022score} to learn the topological ordering of a causal model from observational data. Thanks to the wide applicability of score-matching in machine learning, we also discuss applications to the setting of score-based generative modeling. Our main contributions are:

\begin{enumerate}
    \item We provide the analysis of sample complexity bound for the problem of score function estimation in causal discovery for non-linear additive Gaussian noise models which has a convergence rate of log n/n with respect to the number of data. Importantly, our results require only mild additional assumptions, namely that the non-linear relationships among the causal variables are bounded and that the score function is Lipschitz. To the best of our knowledge, this is the first work to provide sampling complexity bounds for this problem.
    \item We provide the first analysis of the state-of-the-art topological ordering-based causal discovery method SCORE~\citep{rolland2022score} and provide a correctness guarantee for the obtained topological order. Our results demonstrate that the algorithm's error rate converges linearly with respect to the number of training data. Additionally, we establish a connection between the algorithm's error rate and the average second derivative (curvature) of the non-linear relationships among the causal variables, discussing the impact of the causal model's inherent characteristics on the algorithm's error rate in identification.
    \item We present sample complexity bounds for the score function estimation problem in the standard score-based generative modeling method, ScoreSDE~\citep{song2021scorebased}. In contrast to previous results~\citep{chen2023score}, our bounds do not rely on the assumption of low-dimensional input data, and we extend the applicability of the model from a specific encoder-decoder network architecture to a general deep ReLU neural network.
\end{enumerate}

\looseness=-1\textbf{High-level motivation and background:}  Causal discovery and causal inference refer to the process of inferring causation from data and reasoning about the effect of interventions. They are highly relevant in fields such as economics~\citep{varian2016causal}, biology~\citep{sachs2005causal}, and healthcare~\citep{sanchez2022causal}. In particular, some causal discovery methods aim to recover the causal structure of a problem solely based on observational data. 

The causal structure is typically represented as a directed acyclic graph (DAG), where each node is associated with a random variable, and each edge represents a causal mechanism between two variables. Learning such a model from data is known to be NP-hard~\citep{chickering1996learning}. Traditional approaches involve testing for conditional independence between variables or optimizing goodness-of-fit measures to search the space of possible DAGs. However, these greedy combinatorial optimization methods are computationally expensive and difficult to extend to high-dimensional settings.

\looseness=-1An alternative approach is to reframe the combinatorial search problem as a topological ordering task~\citep{teyssier2012ordering, solus2021consistency, wang2021ordering, rolland2022score, montagna2023scalable,montagna2023causal,sanchez2023diffusion}, where nodes are ordered from leaf to root. This can significantly speed up the search process in the DAG space. Once a topological ordering is found, a feature selection algorithm can be used to prune potential causal relations between variables, resulting in a DAG. 

Recently,~\citet{rolland2022score} proposed the SCORE algorithm, which utilizes the Jacobian of the score function to perform topological ordering. By identifying which elements of the Jacobian matrix of the score function remain constant across all data points, leaf nodes can be iteratively identified and removed. This approach provides a systematic way to obtain the topological ordering and infer the causal relations within the entire model. This method has achieved state-of-the-art results on multiple tasks~\citet{rolland2022score} and has been extended to improve scalability~\cite{montagna2023scalable} also using diffusion models~\cite{sanchez2023diffusion} and to non-Gaussian noise~\cite{montagna2023causal}. Interestingly, these approaches separate the concerns of statistical estimation of the score function from the causal assumption used to infer the graph (e.g., non-linear mechanisms and additive Gaussian noise). This opens an opportunity to study the convergence properties of these algorithms in the finite data regime, which is generally under-explored in the causal discovery literature. In fact, if we had error bounds on the score estimate without additional complications from causal considerations, we could study their downstream effect when the score is used for causal discovery.

Unfortunately, this is far from trivial as the theoretical research on score matching lags behind its empirical success and progress would have far-reaching implications. Even beyond causal discovery, error bounds on the estimation of the score functions would be useful for score-based generative modeling (SGM)~\citep{song2019generative,song2021scorebased}. These have achieved state-of-the-art performance in various tasks, including image generation~\citep{dhariwal2021diffusion} and audio synthesis~\citep{kong2020diffwave}. There has been significant research investigating whether accurate score estimation implies that score-based generative modeling provably converges to the true data distribution in realistic settings~\citep{chen2023sampling,lee2022convergence,lee2023convergence}. However, the error bound of score function estimation in the context of score-based generative modeling remains an unresolved issue due to the non-convex training dynamics of neural network optimization.

\textbf{Notations:} \ We use the shorthand $ [n]:= \{1,2,\dots, n \}$ for a positive integer $n$. We denote by $a(n) \lesssim b(n)$: there exists a positive constant $c$ independent of $n$ such that $a(n) \leqslant c b(n)$. The Gaussian distribution is $\mathcal{N}(\mu, \sigma^2)$ with the $\mu$ mean and the $\sigma^2$ variance. We follow the standard Bachmann–Landau notation in complexity theory e.g., $\mathcal{O}$, $o$, $\Omega$, and $\Theta$ for order notation. Due to space constraints, a detailed notation is deferred to \cref{sec:symbols_and_notations}.

\section{Preliminaries}
\label{sec:problem_setting}

As this paper concerns topics in score matching estimation, diffusion models, neural network theory, and causal discovery, we first introduce the background and problem setting of our work. 

\subsection{Score matching}

For a probability density function $p(\bm{x})$, we call the score function the gradient of the log density with respect to the data $\bm{x}$. To estimate the score function $\nabla\log p(\bm{x})$, we can minimize the $\ell_2$ loss over the function space $\mathcal{S}$.
\begin{equation*}
    \min_{\bm{s} \in \mathcal{S}} \mathbb{E}_{p} [\left \| \bm{s}(\bm{x}) - \nabla\log p(\bm{x}) \right \|^2]\,,\quad \hat{\bm{s}}= \arg \min_{\bm{s}\in \mathcal{S}} \mathbb{E}_{p} [\left \| \bm{s}(\bm{x}) - \nabla\log p(\bm{x}) \right \|^2]\,.
\end{equation*}

The corresponding objective function to be minimized is the expected squared error between the true score function and the neural network:
\begin{equation}
    J_{\text{ESM}}(\bm{s}, p(\bm{x})) = \mathbb{E}_{p(\bm{x})} \bigg[\frac{1}{2}\left \| \bm{s}(\bm{x}) - \frac{\partial \log p(\bm{x})}{\partial \bm{x}}  \right \|^2\bigg]\,,
\label{eq:J_ESM}
\end{equation}
We refer to this formulation as explicit score matching (ESM).

Denoising score matching (DSM) is proposed by ~\citet{vincent2011connection} to convert the inference of the score function in ESM into the inference of the random noise and avoid the computing of the second derivative. For the sampled data $\bm{x}$, $\hat{\bm{x}}$ is obtained by adding unit Gaussian noise to $\bm{x}$. i.e. $\hat{\bm{x}} = \bm{x} + \bm{\epsilon}, \ \bm{\epsilon} \sim \mathcal{N}(0, \sigma^2 \bm{I})$.  We can derive the conditional probability distribution and its score function:
\begin{equation*}
    p(\hat{\bm{x}}|\bm{x}) = \frac{1}{(2\pi)^{d/2}\sigma^d}\exp(-\frac{\left \| \bm{x} - \hat{\bm{x}} \right \|^2}{2\sigma^2})\,,\quad\quad \frac{\partial \log p(\hat{\bm{x}}|\bm{x})}{\partial \hat{\bm{x}}} = \frac{\bm{x} - \hat{\bm{x}}}{\sigma^2}\,.
\end{equation*}
Then the DSM is defined by:
\begin{equation}
    J_{\text{DSM}}(\bm{s}, p(\bm{x},\hat{\bm{x}})) = \mathbb{E}_{p(\bm{x},\hat{\bm{x}})} \bigg[\frac{1}{2}\left \| \bm{s}\big(\hat{\bm{x}} - \frac{\partial \log p(\hat{\bm{x}}|\bm{x})}{\partial \hat{\bm{x}}} \big) \right \|^2\bigg] = \mathbb{E}_{p(\bm{x},\hat{\bm{x}})} \bigg[\frac{1}{2}\left \| \bm{s}(\hat{\bm{x}}) - \frac{\bm{x} - \hat{\bm{x}}}{\sigma^2}  \right \|^2\bigg]\,.
\label{eq:J_DSM}
\end{equation}

\citet{vincent2011connection} have proven that minimizing DSM is equivalent to minimizing ESM and does not depend on the particular form of $p(\hat{\bm{x}}|\bm{x})$ or $p(\bm{x})$.

\subsection{Neural network and function space}

In this work, we consider a standard depth-$L$ width-$m$ fully connected ReLU neural network. Formally, we define a DNN with the output $\bm{s}_l(\bm{x})$ in each layer
\begin{equation}
\begin{matrix}
\bm{s}_l(\bm{x}) \!=\! \left\{\begin{matrix}
\bm{x}  & l=0\,,\\
\phi(\langle \bm{W}_l, \bm{s}_{l-1}(\bm{x}) \rangle) & 1\!\leq\! l \!\leq\! L\!-\!1, \\
\left \langle \bm{W}_L, \bm{s}_{L-1}(\bm{x}) \right \rangle & l=L\,,
\end{matrix}\right. \\
\end{matrix}
\label{eq:network}
\end{equation}
where the input is $\bm{x}\in \mathbb{R}^d$, the output is $\bm{s}_L(\bm x) \in \mathbb{R}^d$, the weights of the neural networks are $\bm{W}_1 \in \mathbb{R}^{m \times d} $, $\bm{W}_l \in \mathbb{R}^{m\times m} $, $l = 2,\dots ,L-1$ and $\bm{W}_L \in \mathbb{R}^{d \times m}$. The neural network parameters formulate the tuple of weight matrices $\bm{W} := \{ \bm{W}_i \}_{i=1}^L \in  \{ \mathbb{R}^{m\times d} \times (\mathbb{R}^{m\times m})^{L-2}\times \mathbb{R}^{d\times m} \}$. The $\mathcal{S}$ denotes the function space of~\cref{eq:network}.

The $\phi = \max(0,x)$ is the ReLU activation function. According to the property $\phi(x) = x\phi^{\prime}(x)$ of ReLU, we have $\bm{s}_l = \bm{D}_{l}\bm{W}_l\bm{s}_{l-1}$, where $\bm{D}_{l}$ is a diagonal matrix defined as below.

\begin{definition}[Diagonal sign matrix]
\label{def:diagonal_sign_matrix}
For $l\in [L-1]$ and $k \in [m]$, the diagonal sign matrix $\bm{D}_{l}$ is defined as: $(\bm{D}_{l})_{k,k} = 1\left \{ (\bm{W}_l\bm{s}_{l-1})_k \geq 0 \right \} $.
\end{definition}

\paragraph{Initialization:} We make the standard random Gaussian initialization $[\bm{W}_l]_{i,j}\sim \mathcal{N}(0,\frac{2}{m})$ for $l \in [L-1]$ and $[\bm{W}_L]_{i,j}\sim \mathcal{N}(0,\frac{1}{d})$.

\subsection{Causal discovery}

In this paper, we follow the setting in~\citet{rolland2022score} and consider the following causal model, a random variable $\bm{x} \in \mathbb{R}^d$ is generated by:
\begin{equation}
    x^{(i)} = f_i(\text{PA}_i(\bm{x})) + \epsilon_i\,, \quad i \in [d] \,,
    \label{eq:causal_model}
\end{equation}
where $f_i$ is a non-linear function, $\epsilon_i \sim \mathcal{N}(0, \sigma_i^2)$ and $\text{PA}_i(\bm{x})$ represent the set of parents of $x^{(i)}$ in $\bm{x}$. Then we can write the probability distribution function of $\bm{x}$ as:
\begin{equation}
    p(\bm{x}) = \prod_{i=1}^{d} p(x^{(i)}|\text{PA}_i(\bm{x}))\,.
    \label{eq:causal_model_pdf}
\end{equation}
For such non-linear additive Gaussian noise models~\cref{eq:causal_model},~\citet{rolland2022score} provides~\cref{alg:algorithm_SCORE} to learn the topological order by score matching as follows:
\begin{algorithm}[H]
\caption{SCORE matching causal order search (Adapted from Algorithm 1 in~\citet{rolland2022score})}
\label{alg:algorithm_SCORE}
\begin{algorithmic}
\STATE {\bfseries Input:} training data $\{ (\bm{x}_{(i)})_{i=1}^N \}$.\\
\STATE {\bfseries Initialize:} $\pi = []$, $\text{nodes} = \{1,\ldots,d\}$\\
\FOR {$k=1,\ldots,d$}
    \STATE Estimate the score function $s_{\text{nodes}} = \nabla \log p_{\text{nodes}}$ by deep ReLU network with SGD.
    \STATE Estimate $V_j = \hat{\text{Var}}_{\bm{x}_{\text{nodes}}}\left[\frac{\partial \bm{s}_j(\bm{x})}{\partial \bm{x}^{(j)}}\right]$.
    \STATE $l \leftarrow \text{nodes}[\arg \min_{j} V_j]$
    \STATE $\pi \leftarrow [l, \pi]$
    \STATE $\text{nodes} \leftarrow \text{nodes} - \{l\}$
    \STATE Remove $l$-th element of $\bm{x}$
\ENDFOR
\STATE Get the final DAG by pruning the full DAG associated with the topological order $\pi$.
\end{algorithmic}
\end{algorithm}

\subsection{Score-based generative modeling (SGM)}

In this section, we give a brief overview of SGM following~\citet{song2021scorebased,chen2023sampling}.

\subsubsection{Score-based generative modeling with SDEs}

\paragraph{Forward process:} The success of previous score-based generative modeling methods relies on perturbing data using multiple noise scales, and the proposal of the diffusion model is to expand upon this concept by incorporating an infinite number of noise scales. This will result in the evolution of perturbed data distributions as the noise intensity increases, which will be modeled through a stochastic differential equation (SDE).
\begin{equation}
    \mathrm{d}\bm{x}_t = \bm{f}(\bm{x}_t, t)\mathrm{d} t + g_t \mathrm{d}\bm{w},\quad \bm{x}_0\sim p_0\,.
\label{eq:forward_sde}
\end{equation}
The expression describes $\bm{x}_t$, where the standard Wiener process (also known as Brownian motion) is denoted as $\bm{w}$, the drift coefficient of $\bm{x}_t$ is represented by a vector-valued function called $\bm{f}$, and the diffusion coefficient of $\bm{x}_t$ is denoted as $g_t$, a scalar function. In this context, we will refer to the probability density of $\bm{x}_t$ as $p_t$, and the transition kernel from $\bm{x}_s$ to $\bm{x}_t$ as $p_{st}(\bm{x}_t|\bm{x}_s)$, where $0 \leq s < t \leq T$. The Ornstein–Uhlenbeck (OU) process is a Gaussian process that is both time-homogeneous and a Markov process. It is distinct in that its stationary distribution is equivalent to the standard Gaussian distribution $\gamma^d$ on $\mathbb{R}^d$.

\paragraph{Reverse process:} We can obtain samples of $\bm{x}_0\sim p_0^{\text{SDE}}$ by reversing the process starting from samples of $\bm{x}_T\sim p_T^{\text{SDE}}$. An important finding is that the reversal of a diffusion process is a diffusion process as well. It operates in reverse time and is described by the reverse-time SDE:
\begin{equation}
    \mathrm{d}\bm{x}_t = \bigg(\bm{f}(\bm{x}_t, t) - g_t^2\nabla_{\bm{x}}\log p_t(\bm{x}_t)\bigg)\mathrm{d} t + g_t \mathrm{d}\bm{\overline{w}}\,.
\label{eq:reverse_sde}
\end{equation}
When time is reversed from $T$ to $0$, $\bm{\overline{w}}$ is a standard Wiener process with an infinitesimal negative timestep of $\mathrm{d} t$. The reverse diffusion process can be derived from~\cref{eq:reverse_sde} once the score of each marginal distribution, $\nabla\log p_t(\bm{x}_t)$, is known for all $t$. By simulating the reverse diffusion process, we can obtain samples from $p_0^{\text{SDE}}$.

\paragraph{Some special settings:} In order to simplify the writing of symbols and proofs, in this work we choose that $\bm{f}(\bm{x}_t, t) = -\frac{1}{2}\bm{x}_t$ and $g(t) = 1$ which has been widely employed in prior research~\citep{chen2023score, chen2023sampling, de2021diffusion} for theoretical analysis in Ornstein–Uhlenbeck process in score-based generative modeling.

\subsubsection{Score matching in diffusion model}

We aim to minimize the equivalent objective for score matching:
\begin{equation*}
    \min_{\bm{s} \in \mathcal{S}} \int_{0}^{T} w(t) \mathbb{E}_{\bm{x}_0\sim p_0}\bigg[\mathbb{E}_{\bm{x}_t\sim p_{0t} (\bm{x}_t|\bm{x}_0)}\big[\left \| \nabla_{\bm{x}_t} \log p_{0t} (\bm{x}_t|\bm{x}_0) - \bm{s}(\bm{x}_t,t)\right \|_2^2 \big]\bigg]\mathrm{d}t\,.
\end{equation*}

The transition kernel has an analytical form $\nabla_{\bm{x}_t} \log p_{0t} (\bm{x}_t|\bm{x}_0) = -\frac{\bm{x}_t-\alpha(t)\bm{x}_0}{h(t)}$, where $\alpha(t) = e^{-\frac{t}{2}}$ and $h(t) = 1- \alpha(t)^2 = 1 - e^{-t}$.

The empirical score matching loss is:
\begin{equation}
    \min_{\bm{s} \in \mathcal{S}} \hat{\mathcal{L}}(\bm{s}) = \frac{1}{n}\sum_{i=1}^{n}\ell(\bm{x}_{(i)};\bm{s})\,,
\label{eq:esm_diffusion}
\end{equation}
where the loss function $\ell(\bm{x}_{(i)};\bm{s})$ is defined as:
\begin{equation*}
    \ell(\bm{x}_{(i)};\bm{s}) = \frac{1}{T-t_0}\int_{t_0}^{T} \mathbb{E}_{\bm{x}_t\sim p_{0t} (\bm{x}_t|\bm{x}_0 = \bm{x}_{(i)})}\big[\left \| \nabla_{\bm{x}_t} \log p_{0t} (\bm{x}_t|\bm{x}_0 = \bm{x}_{(i)}) - \bm{s}(\bm{x}_t,t)\right \|_2^2 \big]\mathrm{d}t\,.
\end{equation*}
Here we choose $w(t) = \frac{1}{T-t_0}$, and we define the expected loss $\mathcal{L}(\cdot) = \mathbb{E}_{\bm{x}\sim p_0}[\hat{\mathcal{L}}(\cdot)]$.

\section{Theoretical results for causal discovery}
In this section, we state the main theoretical results of this work. We present the assumptions on non-linear additive Gaussian noise causal models in~\cref{ssec:assumptions}. Then, we present the sample complexity bound for score matching in causal discovery in~\cref{ssec:error_bound_sm_causal}. In~\cref{ssec:error_bound_order_causal} we provide the upper bound on the error rate for causal discovery using the~\cref{alg:algorithm_SCORE}. The full proofs of~\cref{thm:score_bound_causal} and~\ref{thm:topological_order_bound_causal} are deferred to~\cref{sec:proof_error_bound_sm_causal} and~\ref{sec:proof_error_bound_order_causal}, respectively.

\subsection{Assumptions}
\label{ssec:assumptions}

\begin{assumption}[Lipschitz property of score function]
\label{assumption:lipschitz}
The score function $\nabla \log p(\cdot)$ is $1$-Lipschitz.
\end{assumption}
{\bf Remark:} The Lipschitz property of the score function is a standard assumption commonly used in the existing literature~\citep{block2020generative, lee2022convergence, chen2023sampling, chen2023score}. However, for causal discovery, this assumption limits the family of mechanisms that we can cover.

\begin{assumption}[Structural assumptions of causal model]
\label{assumption:structural_assumptions}
Let $p$ be the probability density function of a random variable $\bm{x}$ defined via a non-linear additive Gaussian noise model~\cref{eq:causal_model}. Then, $\forall i \in [d]$ the non-linear function is bounded, $\left | f_i \right | \leq C_i$. And $\ \forall i,j\in [d]$, if $j$ is one of the parents of $i$, i.e. $x^{(j)} \Rightarrow x^{(i)}$, then there exist a constant $C_m$ that satisfy:
\begin{equation*}
\mathbb{E}_{p(\bm{x})}\bigg( \frac{\partial^2 f_i(\text{PA}_i(\bm{x}))}{\partial x^{(j) 2}}^2 \bigg) \geq C_m \sigma_i^2\,.
\end{equation*}
\end{assumption}

{\bf Remark:} This is a novel assumption that we introduce, relating the average second derivative of a mechanism (related to its curvature) to the noise variance of the child variable. This will play a crucial yet intuitive role in our error bound: identifiability is easier when there is sufficient non-linearity of a mechanism with respect to the noise of the child variable. Consider the example of a quadratic mechanism, where the second derivative is the leading constant of the polynomial. If this constant is small (e.g., close to zero), the mechanism is almost linear and we may expect that the causal model should be harder to identify. Similarly, if the child variable has a very large variance, one may expect it to be more difficult to distinguish cause from effect, as the causal effect of the parent is small compared to the noise of the child. According to \cref{assumption:structural_assumptions}, we can derive the identified ability margin for leaf nodes and parent nodes.

\begin{lemma}
\label{lemma:structural_assumptions}
If a non-linear additive Gaussian noise model~\cref{eq:causal_model} satisfies~\cref{assumption:structural_assumptions}. Then, $\forall i,j\in [d]$, we have:
\begin{equation*}
    \text{$i$ is a leaf} \Rightarrow \text{Var}\bigg(\frac{\partial s_i(\bm{x})}{\partial x^{(i)}}\bigg) = 0,\ \text{$j$ is not a leaf} \Rightarrow \text{Var}\bigg(\frac{\partial s_j(\bm{x})}{\partial x^{(j)}}\bigg)\geq C_m.
\end{equation*}
\end{lemma}

This lemma intuitively relates our identifiability margin with the decision rule of SCORE~\cite{rolland2022score} to identify leaves. Non-leaf nodes should have the variance of their score Jacobian sufficiently far from zero. As one may expect, we will see in Theorem~\ref{thm:topological_order_bound_causal} that the closer $C_m$ is to zero, the more likely it is that the result of the algorithm will be incorrect given finite samples.  

\subsection{Error bound for score matching in causal discovery}
\label{ssec:error_bound_sm_causal}
We are now ready to state the main result of the score matching in causal discovery. We provide the sample complexity bounds of the explicit score matching~\cref{eq:J_ESM} that using denoising score matching~\cref{eq:J_DSM} in~\cref{alg:algorithm_SCORE} for non-linear additive Gaussian noise models~\cref{eq:causal_model}.

\begin{theorem}
\label{thm:score_bound_causal}
Given a DNN defined by~\cref{eq:network} trained by SGD for minimizing empirical denoising score matching objective. Suppose~\cref{assumption:lipschitz} and~\ref{assumption:structural_assumptions}  are satisfied. For any $\varepsilon \in (0,1)$ and $\delta \in (0,1)$, if $\sigma_i \eqsim \sigma$ and $\frac{C_i}{\sigma_i}\eqsim 1\,,\ \forall i \in [d]$. Then with probability at least $1- 2\delta - 4\exp(-\frac{d}{32}) - 2L\exp(-\Omega(m)) - \frac{1}{nd}$ over the randomness of initialization $ \bm{W}$, noise $\bm{\epsilon}$ and $\epsilon_i$, it holds that:
\begin{equation*}
    J_{\text{ESM}}(\hat{\bm{s}}, p(\bm{x})) \lesssim \frac{\sigma^2 d\log nd}{n\varepsilon^2}\log\frac{\mathcal{N}_{c}(\frac{1}{n}, \mathcal{S})}{\delta}+\frac{1}{n}+d\varepsilon^2\,,
\end{equation*}
where the $\mathcal{N}_{c}(\frac{1}{n}, \mathcal{S})$ is the covering number of the function space $\mathcal{S}$ for deep ReLU neural network.
\end{theorem}

{\bf Remark:} 

{\bf 1):} To the best of our knowledge, our results present the first upper bound on the explicit sampling complexity of score matching for topological ordering~\cref{alg:algorithm_SCORE} in non-linear additive Gaussian noise causal models. This novel contribution provides valuable insights into the efficiency and effectiveness of utilizing score matching for topological ordering in non-linear additive Gaussian noise causal models.

{\bf 2):} By choosing $\varepsilon^2 = \frac{1}{\sqrt{n}}$, the bound is modified to $J_{\text{ESM}}(\hat{\bm{s}}, p(\bm{x})) \lesssim \frac{\sigma^2 d\log nd}{\sqrt{n}}\log\frac{\mathcal{N}_{c}(1/n, \mathcal{S})}{\delta}$. This expression demonstrates that the $\ell_2$ estimation error converges at a rate of $\frac{\log n}{\sqrt{n}}$ when the sample size $n$ is significantly larger than the number of nodes $d$.

{\bf 3):} The bound is also related to the number of nodes $d$, the variance of the noise in denoising score matching $\sigma$ and causal model $\sigma_i$, the covering number of the function space $\mathcal{N}_{c}(\frac{1}{n}, \mathcal{S})$, and the upper bound of the data $C_d$. If these quantities increase, it is expected that the error of explicit score matching will also increase. This is due to the increased difficulty in accurately estimating the score function. 

{\bf 4):} \cref{thm:score_bound_causal} is rooted in the generalization by sampling complexity bound. It is independent of the specific training algorithm used. The results are broadly applicable and can be seamlessly extended to encompass larger batch GD.

Next, we will establish a connection between score matching and the precise identification of the topological ordering.

\subsection{Error bound for topological order in causal discovery}
\label{ssec:error_bound_order_causal}

Based on the previously mentioned sample complexity bound of score matching, we establish an upper bound on the error rate of the topological ordering of the causal model obtained through~\cref{alg:algorithm_SCORE}.

\begin{theorem}
\label{thm:topological_order_bound_causal}
Given a DNN defined by~\cref{eq:network} trained by SGD with a step size $\eta = \mathcal{O}(\frac{1}{\text{poly}(n,L)m \log^2 m})$ for minimizing empirical score matching objective. Then under~\cref{assumption:structural_assumptions}, for $m \geq \text{poly}(n,L)$, with probability at least: 
\begin{equation*}
    1-\exp(-\Theta(d))-(L+1)\exp(-\Theta(m))- 2n\exp(-\frac{nC_m^2 d^2}{2^{4L+5}(\log m)^2 (m^2+d^2)} )\,,
\end{equation*}
over the randomness of initialization $\bm{W}$ and training data that~\cref{alg:algorithm_SCORE} can completely recover the correct topological order of the non-linear additive Gaussian noise model.

\end{theorem}

{\bf Remark:} 

{\bf 1):} The foundation of~\cref{thm:topological_order_bound_causal} rests upon~\cref{thm:score_bound_causal}, it can be seen as an embodiment of applying the upper bound of score matching for causal discovery. To the best of our knowledge, our results provide the first upper bound on the error rate of topological ordering in non-linear additive Gaussian noise causal models using~\cref{alg:algorithm_SCORE}.

{\bf 2):} Considering that when $ m \eqsim d$ and $L \eqsim 1$ the probability degenerates to:
\begin{equation*}
    1-\Theta(e^{-m})- 2n\exp\bigg(-\Theta\big(\frac{nC_m^2 }{(\log m)^2}\big) \bigg)\,.
\end{equation*}
The first term of the error arises due to the initialization of the neural network. As for the second term of the error, if the number of training data $n$ satisfies $\frac{n}{\log n} \gtrsim (\log m)^2$, then it will have that $2n\exp\big(-\Theta\big(\frac{nC_m^2 }{(\log m)^2}\big) \big) \lesssim 1$. This implies that the second term of the error probability exhibits linear convergence towards $0$ when $n$ is sufficiently large. Therefore, when the sample size $\frac{n}{\log n} \gtrsim (\log m)^2$, the contribution of the second term to the full error becomes negligible.

{\bf 3):} The theorem reveals that a smaller value of the constant $C_m$ increases the probability of algorithm failure. This observation further confirms our previous statement that a smaller average second derivative of the nonlinear function makes it more challenging to identify the causal relationship in the model. Additionally, when the causal relationship is linear, our theorem does not provide any guarantee for the performance of~\cref{alg:algorithm_SCORE}.

{\bf 4):} Consider the two variables case. If a child node is almost a deterministic function of its parents, the constant $C_m$ can take on arbitrarily large values, according to \cref{assumption:structural_assumptions}. Consequently, the second term of the error probability, $2n\exp\big(-\Theta\big(\frac{nC_m^2 }{(\log m)^2}\big) \big)$, tends to zero. This implies that the errors in \cref{alg:algorithm_SCORE} are primarily caused by the random initialization of the neural network. The identifiability of this setting is consistent with classical results~\cite{daniuvsis2010inferring,janzing2015justifying}. Intuitively, as long as the non-linearity is chosen independently of the noise of the parent variable\footnote{~\cite{daniuvsis2010inferring,janzing2015justifying} have formalized independence of distribution and function via an information geometric orthogonality condition that refers to a reference distribution (e.g., Gaussian)}, the application of the non-linearity will increase the distance to the reference distribution of the parent variable (in our case Gaussian). Note that for the derivative in Assumption~\cref{assumption:structural_assumptions} to be defined, the parent node cannot be fully deterministic. 

{\bf 5):} Instead of focusing on the kernel regime, we directly cover the more general neural network training. The kernel approach of~\citet{rolland2022score} is a special case of our analysis. The basis of Theorem 2 lies in the proof of SGD/GD convergence of the neural network, These convergence outcomes also apply to BatchGD, as demonstrated in~\citet{jentzen2021convergence}. Hence, Theorem 2 can naturally be expanded to incorporate Batch GD as well.

{\bf Proof sketch:} The proof of~\cref{thm:topological_order_bound_causal} can be divided into three steps. The first and most important step is to derive the upper bound of $\frac{\partial s_i(\bm{x})}{\partial x^{(i)}}$. Here, we utilize the properties of deep ReLU neural networks to derive the distribution relationship between features of adjacent layers, then accumulate them and combine it with the properties of Gaussian initialization, yielding the upper bound for $\frac{\partial s_i(\bm{x})}{\partial x^{(i)}}$. The second step is to use the upper bound of $\frac{\partial s_i(\bm{x})}{\partial x^{(i)}}$ obtained in the first step combined with the concentration inequality to derive the upper bound of the error of $\text{Var}\big(\frac{\partial s_i(\bm{x})}{\partial x^{(i)}}\big)$. The third step is to compare the upper bound in the second step with~\cref{lemma:structural_assumptions} to obtain the probability of successfully selecting leaf nodes in each step. After accumulation, we can obtain the probability that~\cref{alg:algorithm_SCORE} can completely recover the correct topological order of the non-linear additive Gaussian noise model.

\section{Theoretical results for score-based generative modeling (SGM)}

In this section, we present the additional assumption required for the theoretical analysis of score matching in score-based generative modeling. Then, we provide the sample complexity bound associated with score matching in this framework. The full proof in this section is deferred to~\cref{sec:proof_error_bound_diffusion}.

\begin{assumption}[Bounded data]
\label{assumption:bounded_data}
We assume that the input data satisfy $\| \bm{x} \|_{2} \leq C_d \,,\quad \bm{x} \sim p_0$.
\end{assumption}

{\bf Remark:} Bounded data is standard in deep learning theory and also commonly used in practice~\citep{du2018gradient, du2019gradient, allen2019convergence, oymak2020toward, pmlr-v119-malach20a}. 

\begin{theorem}
\label{thm:score_bound_diffusion}
Given a DNN defined by~\cref{eq:network} trained by SGD for minimizing empirical denoising score matching loss~\cref{eq:esm_diffusion}. Suppose~\cref{assumption:lipschitz} and~\ref{assumption:bounded_data} are satisfied. For any $\varepsilon \in (0,1)$ and $\delta \in (0,1)$. Then with probability at least $1- 2\delta - 2L\exp(-\Omega(m))$ over the randomness of initialization $ \bm{W}$ and noise $\bm{\epsilon}$ in denoising score matching, it holds:
\begin{equation*}
    \frac{1}{T-t_0}\int_{t_0}^{T} \left \| \nabla \log p_t(\cdot) - \hat{\bm{s}}(\cdot,t)\right \|_{\ell^2(p_t)}^2 \mathrm{d}t \lesssim \frac{1}{n\varepsilon^2}\bigg(\frac{d(T-\log(t_0))}{T-t_0}+C_d^2\bigg)\log\frac{\mathcal{N}_{c}(\frac{1}{n}, \mathcal{S})}{\delta}+\frac{1}{n} + d\varepsilon^2\,,
\end{equation*}
where the $\mathcal{N}_{c}(\frac{1}{n}, \mathcal{S})$ is the covering number of the function space $\mathcal{S}$ for deep ReLU neural network.
\end{theorem}

{\bf Remark:} 

{\bf 1):} \cref{thm:score_bound_diffusion} and~\cref{thm:score_bound_causal} study similar problems between causal discovery and score-based generative modeling and share similar techniques drawn from statistical learning theory and deep learning theory. These two domains are connected by a common theoretical foundation centered on the upper bound of score matching.

{\bf 2):} Our result extends the results for score matching in diffusion models presented in~\citet{chen2023score} which rested on the assumption of low-dimensional data structures, employing this to decompose the score function and engineer specialized network architectures for the derivation of the upper bound. Our work takes a distinct route. Our conclusions are based on the general deep ReLU neural network instead of a specific encoder-decoder network and do not rely on the assumptions of low-dimensional data used in~\citet{chen2023score}. We harness the inherent traits and conventional techniques of standard deep ReLU networks to directly deduce the upper error bound. This broader scope allows for a more comprehensive understanding of the implications and applicability of score-based generative modeling in a wider range of scenarios.

{\bf 3):} Similar to~\cref{thm:score_bound_causal}, by choosing $\varepsilon^2 = \frac{1}{\sqrt{n}}$, we can obtain the best bound $\frac{1}{T-t_0}\int_{t_0}^{T} \left \| \nabla \log p_t(\cdot) - \hat{\bm{s}}(\cdot,t)\right \|_{\ell^2(p_t)}^2 \mathrm{d}t \lesssim \frac{1}{\sqrt{n}}\bigg(\frac{d(T-\log(t_0))}{T-t_0}+C_d^2\bigg)\log\frac{\mathcal{N}_{c}(\frac{1}{n}, \mathcal{S})}{\delta}$. This expression demonstrates that the $\ell_2$ estimation error converges at a rate of $\frac{1}{\sqrt{n}}$ when the sample size $n$ is significantly larger than the dimensionality $d$ and time steps $T$.

{\bf 4):} The bound is also related to the data dimension $d$, the variance of the noise in denoising score matching $\sigma$, the covering number of the function space $\mathcal{N}_{c}(\frac{1}{n}, \mathcal{S})$, and the upper bound of the data $C_d$. If these quantities increase, it is expected that the error of explicit score matching will also increase. This is due to the increased difficulty in accurately estimating the score function.

{\bf 5):} When $t_0=0$, the theorem lacks meaning. However, when $ T \gg t_0 \eqsim 1$, the bound simplifies to $\frac{d+C_d^2}{\sqrt{n}}\log\frac{\mathcal{N}_{c}(\frac{1}{n}, \mathcal{S})}{\delta}$. This indicates that when $T$ is sufficiently large, the loss estimated by the score function in the diffusion model becomes independent of time steps $T$.

{\bf 6):} Similar to~\cref{thm:score_bound_causal}, the result of~\cref{thm:score_bound_diffusion} is also broadly applicable and can be seamlessly extended to encompass larger batch GD.

\section{Numerical evidence}
\label{sec:experiment}

We conducted a series of experiments to validate the theoretical findings presented in the paper. We took inspiration from the code provided in\citet{rolland2022score} and employed the structural Hamming distance (SHD) between the generated output and the actual causal graph to assess the outcomes. The ensuing experimental outcomes for SHD, vary across causal model sizes $d$, sample sizes $n$, and $C_m$. The experimental results are shown in~\cref{tab:exp_diff_cm,tab:exp_diff_n,tab:exp_diff_d}

\begin{table}[H]
\scriptsize  
\centering
\begin{threeparttable}
\caption{Fixed model size $d=100$ and the number of sampling $n=100$, SHD results of causal discovery using Algorithm 1 for different $C_m$ values ($10$ runs).}
\label{tab:exp_diff_cm}
\setlength{\tabcolsep}{2.5mm}{
\begin{tabular}{c|c|c|c|c|c}
    \toprule[1pt]
    $C_m$ & $1$ & $2$ & $4$ & $8$ & $16$ \\
    \midrule
    SHD & $2941.0\pm29.5$  & $2905.7 \pm 50.8$ & $2900.6 \pm 80.8$ & $2637.1 \pm 200.4$ & $1512.4 \pm 283.6$\\
    \midrule[1pt]
    $C_m$ & $32$ & $64$ & $128$ & $256$ & $512$\\
    \midrule
    SHD & $413.9 \pm 93.4$ & $55.0 \pm 16.0$ & $ 23.9 \pm 4.6$ & $21.2 \pm 5.0$ & $13.8 \pm 1.8$  \\
\bottomrule[1pt]
\end{tabular}}
\end{threeparttable}
\end{table}

\begin{table}[H]
\scriptsize  
\centering
\begin{threeparttable}
\caption{Fixed model size $d=10$ and $C_m = 1$, SHD results of causal discovery using Algorithm 1 for the different number of sampling $n$ ($10$ runs).}
\label{tab:exp_diff_n}
\setlength{\tabcolsep}{2.5mm}{
\begin{tabular}{c|c|c|c|c|c|c|c}
    \toprule[1pt]
    $n$ & $5$ & $10$ & $20$ & $40$ & $80$ & $100$ & $160$\\
    \midrule
    SHD  & $31.7\pm2.1$  & $27.8\pm4.1$ & $23.3\pm2.7$ & $23.0\pm4.0$ & $18.4 \pm 3.3$ & $16.5 \pm 3.4$ & $13.0\pm 4.0$\\
\bottomrule[1pt]
\end{tabular}}
\end{threeparttable}
\end{table}

\begin{table}[H]
\scriptsize  
\centering
\begin{threeparttable}
\caption{Fixed the number of sampling $n=10$ and $C_m = 1$, SHD results of causal discovery using Algorithm 1 for the different model size $d$ ($10$ runs).}
\label{tab:exp_diff_d}
\setlength{\tabcolsep}{2.5mm}{
\begin{tabular}{c|c|c|c|c|c|c}
    \toprule[1pt]
    $d$ & $5$ & $10$ & $20$ & $40$ & $80$ & $100$ \\
    \midrule
    SHD  & $4.5 \pm 2.0$  & $29.6 \pm 2.2$ & $ 124.3 \pm 4.6 $ & $522.8 \pm 11.6$ & $1965.4 \pm 18.7$ & $2923.7 \pm 38.5$ \\
\bottomrule[1pt]
\end{tabular}}
\end{threeparttable}
\end{table}

Analyzing the experimental outcomes, we find a notable pattern: higher values of $C_m$, augmented sample sizes $n$, and reduced model size $d$ all contribute to the performance of~\cref{alg:algorithm_SCORE} which is consistent with the insights from~\cref{thm:topological_order_bound_causal}.
\section{Related Work}
\label{sec:related}

\paragraph{Score matching:} Score Matching was initially introduced by~\citet{hyvarinen2005estimation} and extended to energy-based models by~\citet{song2019generative}. Subsequently,~\citet{vincent2011connection} proposed denoising score matching, which transforms the estimation of the score function for the original distribution into an estimation for the noise distribution, effectively avoiding the need for second derivative computations. Other methods, such as sliced score matching~\citep{song2020sliced}, denoising likelihood score matching~\citep{chao2022denoising}, and kernel-based estimators, have also been proposed for score matching. The relationship between score matching and Fisher information~\citep{shao2019bayesian}, as well as Langevin dynamics~\citep{hyvarinen2007connections}, has been explored. On the theoretical side,~\citet{wenliang2020blindness} introduced the concept of "blindness" in score matching, while~\citet{koehler2023statistical} compared the efficiency of maximum likelihood and score matching, although their results primarily focus on exponential family distributions. Our paper, for the first time, analyzes the sample complexity bounds of the score function estimating in causal inference.

\paragraph{Causal discovery:} The application of score methods for causal inference for linear additive models began with \citet{ghoshal2018learning}, which proposed a causal structure recovery method based on topological ordering from the precision matrix (equivalent to the score in that setting). Under certain noise variance assumptions, their method can reliably recover the DAG in polynomial time and sample complexity. 

\looseness=-1In recent years, there have been numerous algorithms developed for causal inference in non-linear additive models. GraNDAG~\citep{lachapelle2019gradient} aims to maximize the likelihood of the observed data under this model and enforces a continuous constraint to ensure the acyclicity of the causal graph~\citet{rolland2022score} proposed a novel approach for causal inference which utilize score matching algorithms as a foundation for topological ordering and then employ sparse regression techniques to prune the DAG. Subsequently,~\citet{montagna2023causal} extended the method to non-Gaussian noise,~\cite{sanchez2023diffusion} proposed to use diffusion models to fit the score function, and~\cite{montagna2023scalable} proposed a new scalable score-based preliminary neighbor search techniques. 

Although advances have been achieved in leveraging machine learning for causal discovery, there is generally a lack of further research on error bounds. Other studies concentrate on broader non-parametric models but depend on various assumptions like faithfulness, restricted faithfulness, or the sparsest Markov representation~\citep{spirtes2000causation, raskutti2018learning, solus2021consistency}. These approaches employ conditional independence tests and construct a graph that aligns with the identified conditional independence relations \citep{zhang2008completeness}.

\looseness=-1\paragraph{Theoretical analysis of score-based generative modeling:} Existing work mainly focuses on two fundamental questions: "How do diffusion models utilize the learned score functions to estimate the data distribution?"~\citep{chen2023sampling, de2021diffusion, de2022convergence, lee2022convergence,lee2023convergence} and "Can neural networks effectively approximate and learn score functions? What are the convergence rate and bounds on the sample complexity?"~\citep{chen2023score}.

Specifically,~\citet{de2021diffusion} and~\citet{lee2022convergence} studied the convergence guarantees of diffusion models under the assumptions that the score estimator is accurate under the $\ell_1$ and $\ell_2$ norms. Concurrently~\citet{chen2023sampling} and~\citet{lee2023convergence} extended previous results to distributions with bounded moments.~\citet{de2022convergence} studied the distribution estimation guarantees of diffusion models for low-dimensional manifold data under the assumption that the score estimator is accurate under the $\ell_1$ or $\ell_2$ norms.

However, these theoretical results rely on the assumption that the score function is accurately estimated, while the estimation of the score function is largely untouched due to the non-convex training dynamics. Recently,~\citet{chen2023score} provided the first sample complexity bounds for score function estimation in diffusion models. However, their result is based on the assumption that the data distribution is supported on a low-dimensional linear subspace and they use a specialized Encoder-Decoder network instead of a general deep neural network. As a result, a complete theoretical picture of score-based generative modeling is still lacking.
\section{Conclusion and Limitations}
\label{sec:conclusion}
In this work, we investigate the sample complexity error bounds of Score Matching using deep ReLU neural networks under two different problem settings: causal discovery and score-based generative modeling. We provide a sample complexity analysis for the estimation of the score function in the context of causal discovery for nonlinear additive Gaussian noise models, with a convergence rate of $\frac{\log n}{n}$. Furthermore, we extend the sample complexity bounds for the estimation of the score function in the ScoreSDE method to general data and achieve a convergence rate of $\frac{1}{n}$. Additionally, we provide an upper bound on the error rate of the state-of-the-art causal discovery method SCORE~\citep{rolland2022score}, showing that the error rate of this algorithm converges linearly with respect to the number of training data.

A core limitation of this work is limiting our results to the Gaussian noise assumption. In fact, non-linear mechanisms with additive non-gaussian noise are also identifiable under mild additional assumptions~\citep{peters2014causal} and~\cite{montagna2023causal} already extended the score-matching approach of~\cite{rolland2022score} to that setting. Relaxing this assumption would also allow us to apply our bounds to interesting corner cases, such as linear non-gaussian~\citep{ghoshal2018learning}, and non-gaussian deterministic causal relations~\citep{daniuvsis2010inferring,janzing2015justifying}.
It may be possible for this assumption to be relaxed in future work, but we argue that the added challenge, the significant difference in algorithms, and the standalone importance of the non-linear Gaussian case justify our focus.

In addition, we make other assumptions that limit the general applicability of our bounds. In particular, the assumption of the Lipschitz property for the score function imposes a strong constraint on the model space. Further investigating the relationship between the noise, the properties of the nonlinear functions in the causal model~\cref{eq:causal_model}, and the resulting Lipschitz continuity of the score function would be an interesting extension of this work. 
\section*{Acknowledgements}
\label{sec:acks}

We are thankful to the reviewers for providing constructive feedback and Kun Zhang and Dominik Janzing for helpful discussion on the special case of deterministic children. This work was supported by Hasler Foundation Program: Hasler Responsible AI (project number 21043). This work was supported by the Swiss National Science Foundation (SNSF) under grant number 200021$\_$205011. Francesco Locatello did not contribute to this work at Amazon. Corresponding author: Zhenyu Zhu.
\newpage
\bibliography{literature}

\begin{thebibliography}{50}
\providecommand{\natexlab}[1]{#1}
\providecommand{\url}[1]{\texttt{#1}}
\expandafter\ifx\csname urlstyle\endcsname\relax
  \providecommand{\doi}[1]{doi: #1}\else
  \providecommand{\doi}{doi: \begingroup \urlstyle{rm}\Url}\fi

\bibitem[Allen-Zhu et~al.(2019)Allen-Zhu, Li, and Song]{allen2019convergence}
Z.~Allen-Zhu, Y.~Li, and Z.~Song.
\newblock A convergence theory for deep learning via over-parameterization.
\newblock In \emph{International Conference on Machine Learning (ICML)}, 2019.

\bibitem[Block et~al.(2020)Block, Mroueh, and Rakhlin]{block2020generative}
A.~Block, Y.~Mroueh, and A.~Rakhlin.
\newblock Generative modeling with denoising auto-encoders and langevin sampling, 2020.

\bibitem[Chao et~al.(2022)Chao, Sun, Cheng, Lo, Chang, Liu, Chang, Chen, and Lee]{chao2022denoising}
C.-H. Chao, W.-F. Sun, B.-W. Cheng, Y.-C. Lo, C.-C. Chang, Y.-L. Liu, Y.-L. Chang, C.-P. Chen, and C.-Y. Lee.
\newblock Denoising likelihood score matching for conditional score-based data generation.
\newblock In \emph{International Conference on Learning Representations (ICLR)}, 2022.

\bibitem[Chen et~al.(2020)Chen, Liao, Zha, and Zhao]{chen2020distribution}
M.~Chen, W.~Liao, H.~Zha, and T.~Zhao.
\newblock Distribution approximation and statistical estimation guarantees of generative adversarial networks, 2020.

\bibitem[Chen et~al.(2023{\natexlab{a}})Chen, Huang, Zhao, and Wang]{chen2023score}
M.~Chen, K.~Huang, T.~Zhao, and M.~Wang.
\newblock Score approximation, estimation and distribution recovery of diffusion models on low-dimensional data.
\newblock In \emph{International Conference on Machine Learning (ICML)}, 2023{\natexlab{a}}.

\bibitem[Chen et~al.(2023{\natexlab{b}})Chen, Chewi, Li, Li, Salim, and Zhang]{chen2023sampling}
S.~Chen, S.~Chewi, J.~Li, Y.~Li, A.~Salim, and A.~Zhang.
\newblock Sampling is as easy as learning the score: theory for diffusion models with minimal data assumptions.
\newblock In \emph{International Conference on Learning Representations (ICLR)}, 2023{\natexlab{b}}.

\bibitem[Chickering(1996)]{chickering1996learning}
D.~M. Chickering.
\newblock Learning bayesian networks is np-complete.
\newblock \emph{Learning from data: Artificial intelligence and statistics V}, 1996.

\bibitem[Daniu{\v{s}}is et~al.(2010)Daniu{\v{s}}is, Janzing, Mooij, Zscheischler, Steudel, Zhang, and Sch{\"o}lkopf]{daniuvsis2010inferring}
P.~Daniu{\v{s}}is, D.~Janzing, J.~Mooij, J.~Zscheischler, B.~Steudel, K.~Zhang, and B.~Sch{\"o}lkopf.
\newblock Inferring deterministic causal relations.
\newblock In \emph{Uncertainty in Artificial Intelligence}, 2010.

\bibitem[De~Bortoli(2022)]{de2022convergence}
V.~De~Bortoli.
\newblock Convergence of denoising diffusion models under the manifold hypothesis, 2022.

\bibitem[De~Bortoli et~al.(2021)De~Bortoli, Thornton, Heng, and Doucet]{de2021diffusion}
V.~De~Bortoli, J.~Thornton, J.~Heng, and A.~Doucet.
\newblock Diffusion schr{\"o}dinger bridge with applications to score-based generative modeling.
\newblock In \emph{Advances in neural information processing systems (NeurIPS)}, 2021.

\bibitem[Dhariwal and Nichol(2021)]{dhariwal2021diffusion}
P.~Dhariwal and A.~Nichol.
\newblock Diffusion models beat gans on image synthesis.
\newblock In \emph{Advances in neural information processing systems (NeurIPS)}, 2021.

\bibitem[Du et~al.(2019{\natexlab{a}})Du, Lee, Li, Wang, and Zhai]{du2019gradient}
S.~Du, J.~Lee, H.~Li, L.~Wang, and X.~Zhai.
\newblock Gradient descent finds global minima of deep neural networks.
\newblock In \emph{International Conference on Machine Learning (ICML)}, 2019{\natexlab{a}}.

\bibitem[Du et~al.(2019{\natexlab{b}})Du, Zhai, Poczos, and Singh]{du2018gradient}
S.~S. Du, X.~Zhai, B.~Poczos, and A.~Singh.
\newblock Gradient descent provably optimizes over-parameterized neural networks.
\newblock In \emph{International Conference on Learning Representations (ICLR)}, 2019{\natexlab{b}}.

\bibitem[Ghosh(2021)]{Ghosh2021Chisquared}
M.~Ghosh.
\newblock Exponential tail bounds for chisquared random variables.
\newblock \emph{Journal of Statistical Theory and Practice}, 2021.

\bibitem[Ghoshal and Honorio(2018)]{ghoshal2018learning}
A.~Ghoshal and J.~Honorio.
\newblock Learning linear structural equation models in polynomial time and sample complexity.
\newblock In \emph{International Conference on Artificial Intelligence and Statistics (AISTATS)}, 2018.

\bibitem[Hyv{\"a}rinen(2005)]{hyvarinen2005estimation}
A.~Hyv{\"a}rinen.
\newblock Estimation of non-normalized statistical models by score matching.
\newblock \emph{Journal of Machine Learning Research}, 2005.

\bibitem[Hyvarinen(2007)]{hyvarinen2007connections}
A.~Hyvarinen.
\newblock Connections between score matching, contrastive divergence, and pseudolikelihood for continuous-valued variables.
\newblock \emph{IEEE Transactions on neural networks}, 2007.

\bibitem[Janzing et~al.(2015)Janzing, Steudel, Shajarisales, and Sch{\"o}lkopf]{janzing2015justifying}
D.~Janzing, B.~Steudel, N.~Shajarisales, and B.~Sch{\"o}lkopf.
\newblock Justifying information-geometric causal inference.
\newblock \emph{Measures of Complexity: Festschrift for Alexey Chervonenkis}, 2015.

\bibitem[Jentzen and Kr{\"o}ger(2021)]{jentzen2021convergence}
A.~Jentzen and T.~Kr{\"o}ger.
\newblock Convergence rates for gradient descent in the training of overparameterized artificial neural networks with biases, 2021.

\bibitem[Koehler et~al.(2023)Koehler, Heckett, and Risteski]{koehler2023statistical}
F.~Koehler, A.~Heckett, and A.~Risteski.
\newblock Statistical efficiency of score matching: The view from isoperimetry.
\newblock In \emph{International Conference on Learning Representations (ICLR)}, 2023.

\bibitem[Kong et~al.(2021)Kong, Ping, Huang, Zhao, and Catanzaro]{kong2020diffwave}
Z.~Kong, W.~Ping, J.~Huang, K.~Zhao, and B.~Catanzaro.
\newblock Diffwave: A versatile diffusion model for audio synthesis.
\newblock In \emph{International Conference on Learning Representations (ICLR)}, 2021.

\bibitem[Lachapelle et~al.(2021)Lachapelle, Brouillard, Deleu, and Lacoste-Julien]{lachapelle2019gradient}
S.~Lachapelle, P.~Brouillard, T.~Deleu, and S.~Lacoste-Julien.
\newblock Gradient-based neural dag learning.
\newblock In \emph{International Conference on Learning Representations (ICLR)}, 2021.

\bibitem[Lee et~al.(2022)Lee, Lu, and Tan]{lee2022convergence}
H.~Lee, J.~Lu, and Y.~Tan.
\newblock Convergence for score-based generative modeling with polynomial complexity.
\newblock In \emph{Advances in neural information processing systems (NeurIPS)}, 2022.

\bibitem[Lee et~al.(2023)Lee, Lu, and Tan]{lee2023convergence}
H.~Lee, J.~Lu, and Y.~Tan.
\newblock Convergence of score-based generative modeling for general data distributions.
\newblock In \emph{International Conference on Algorithmic Learning Theory}, 2023.

\bibitem[Malach et~al.(2020)Malach, Yehudai, Shalev-Schwartz, and Shamir]{pmlr-v119-malach20a}
E.~Malach, G.~Yehudai, S.~Shalev-Schwartz, and O.~Shamir.
\newblock Proving the lottery ticket hypothesis: Pruning is all you need.
\newblock In \emph{International Conference on Machine Learning (ICML)}, 2020.

\bibitem[Montagna et~al.(2023{\natexlab{a}})Montagna, Noceti, Rosasco, Zhang, and Locatello]{montagna2023causal}
F.~Montagna, N.~Noceti, L.~Rosasco, K.~Zhang, and F.~Locatello.
\newblock Causal discovery with score matching on additive models with arbitrary noise.
\newblock In \emph{CLeaR}, 2023{\natexlab{a}}.

\bibitem[Montagna et~al.(2023{\natexlab{b}})Montagna, Noceti, Rosasco, Zhang, and Locatello]{montagna2023scalable}
F.~Montagna, N.~Noceti, L.~Rosasco, K.~Zhang, and F.~Locatello.
\newblock Scalable causal discovery with score matching.
\newblock In \emph{CLeaR}, 2023{\natexlab{b}}.

\bibitem[Nguyen et~al.(2021)Nguyen, Mondelli, and Montufar]{Nguyen2021eigenvalue}
Q.~Nguyen, M.~Mondelli, and G.~F. Montufar.
\newblock Tight bounds on the smallest eigenvalue of the neural tangent kernel for deep relu networks.
\newblock In \emph{International Conference on Machine Learning (ICML)}, 2021.

\bibitem[Oymak and Soltanolkotabi(2020)]{oymak2020toward}
S.~Oymak and M.~Soltanolkotabi.
\newblock Toward moderate overparameterization: Global convergence guarantees for training shallow neural networks.
\newblock \emph{{IEEE} Journal on Selected Areas in Information Theory}, 2020.

\bibitem[Peters et~al.(2014)Peters, Mooij, Janzing, and Sch{\"o}lkopf]{peters2014causal}
J.~Peters, J.~M. Mooij, D.~Janzing, and B.~Sch{\"o}lkopf.
\newblock Causal discovery with continuous additive noise models.
\newblock \emph{Journal of Machine Learning Research}, 2014.

\bibitem[Raskutti and Uhler(2018)]{raskutti2018learning}
G.~Raskutti and C.~Uhler.
\newblock Learning directed acyclic graph models based on sparsest permutations.
\newblock \emph{Stat}, 2018.

\bibitem[Rolland et~al.(2022)Rolland, Cevher, Kleindessner, Russell, Janzing, Sch{\"o}lkopf, and Locatello]{rolland2022score}
P.~Rolland, V.~Cevher, M.~Kleindessner, C.~Russell, D.~Janzing, B.~Sch{\"o}lkopf, and F.~Locatello.
\newblock Score matching enables causal discovery of nonlinear additive noise models.
\newblock In \emph{International Conference on Machine Learning (ICML)}, 2022.

\bibitem[Sachs et~al.(2005)Sachs, Perez, Pe'er, Lauffenburger, and Nolan]{sachs2005causal}
K.~Sachs, O.~Perez, D.~Pe'er, D.~A. Lauffenburger, and G.~P. Nolan.
\newblock Causal protein-signaling networks derived from multiparameter single-cell data.
\newblock \emph{Science}, 2005.

\bibitem[Sanchez et~al.(2022)Sanchez, Voisey, Xia, Watson, O’Neil, and Tsaftaris]{sanchez2022causal}
P.~Sanchez, J.~P. Voisey, T.~Xia, H.~I. Watson, A.~Q. O’Neil, and S.~A. Tsaftaris.
\newblock Causal machine learning for healthcare and precision medicine.
\newblock \emph{Royal Society Open Science}, 2022.

\bibitem[Sanchez et~al.(2023)Sanchez, Liu, O'Neil, and Tsaftaris]{sanchez2023diffusion}
P.~Sanchez, X.~Liu, A.~Q. O'Neil, and S.~A. Tsaftaris.
\newblock Diffusion models for causal discovery via topological ordering.
\newblock In \emph{International Conference on Learning Representations (ICLR)}, 2023.

\bibitem[Shalev-Shwartz and Ben-David(2014)]{shalev2014understanding}
S.~Shalev-Shwartz and S.~Ben-David.
\newblock \emph{Understanding machine learning: From theory to algorithms}.
\newblock Cambridge university press, 2014.

\bibitem[Shao et~al.(2019)Shao, Jacob, Ding, and Tarokh]{shao2019bayesian}
S.~Shao, P.~E. Jacob, J.~Ding, and V.~Tarokh.
\newblock Bayesian model comparison with the hyv{\"a}rinen score: Computation and consistency.
\newblock \emph{Journal of the American Statistical Association}, 2019.

\bibitem[Solus et~al.(2021)Solus, Wang, and Uhler]{solus2021consistency}
L.~Solus, Y.~Wang, and C.~Uhler.
\newblock Consistency guarantees for greedy permutation-based causal inference algorithms.
\newblock \emph{Biometrika}, 2021.

\bibitem[Song and Ermon(2019)]{song2019generative}
Y.~Song and S.~Ermon.
\newblock Generative modeling by estimating gradients of the data distribution.
\newblock In \emph{Advances in neural information processing systems (NeurIPS)}, 2019.

\bibitem[Song et~al.(2020)Song, Garg, Shi, and Ermon]{song2020sliced}
Y.~Song, S.~Garg, J.~Shi, and S.~Ermon.
\newblock Sliced score matching: A scalable approach to density and score estimation.
\newblock In \emph{Uncertainty in Artificial Intelligence}, 2020.

\bibitem[Song et~al.(2021)Song, Sohl-Dickstein, Kingma, Kumar, Ermon, and Poole]{song2021scorebased}
Y.~Song, J.~Sohl-Dickstein, D.~P. Kingma, A.~Kumar, S.~Ermon, and B.~Poole.
\newblock Score-based generative modeling through stochastic differential equations.
\newblock In \emph{International Conference on Learning Representations (ICLR)}, 2021.

\bibitem[Spirtes et~al.(2000)Spirtes, Glymour, and Scheines]{spirtes2000causation}
P.~Spirtes, C.~N. Glymour, and R.~Scheines.
\newblock \emph{Causation, prediction, and search}.
\newblock MIT press, 2000.

\bibitem[Teyssier and Koller(2012)]{teyssier2012ordering}
M.~Teyssier and D.~Koller.
\newblock Ordering-based search: A simple and effective algorithm for learning bayesian networks, 2012.

\bibitem[Varian(2016)]{varian2016causal}
H.~R. Varian.
\newblock Causal inference in economics and marketing.
\newblock \emph{Proceedings of the National Academy of Sciences}, 2016.

\bibitem[Vershynin(2018)]{vershynin12}
R.~Vershynin.
\newblock \emph{High-Dimensional Probability: An Introduction with Applications in Data Science}.
\newblock Taylor \& Francis, 2018.

\bibitem[Vincent(2011)]{vincent2011connection}
P.~Vincent.
\newblock A connection between score matching and denoising autoencoders.
\newblock \emph{Neural computation}, 2011.

\bibitem[Wang et~al.(2021)Wang, Du, Zhu, Ke, Chen, Hao, and Wang]{wang2021ordering}
X.~Wang, Y.~Du, S.~Zhu, L.~Ke, Z.~Chen, J.~Hao, and J.~Wang.
\newblock Ordering-based causal discovery with reinforcement learning.
\newblock In \emph{International Joint Conferences on Artificial Intelligence (IJCAI)}, 2021.

\bibitem[Wenliang and Kanagawa(2020)]{wenliang2020blindness}
L.~K. Wenliang and H.~Kanagawa.
\newblock Blindness of score-based methods to isolated components and mixing proportions, 2020.

\bibitem[Zhang(2008)]{zhang2008completeness}
J.~Zhang.
\newblock On the completeness of orientation rules for causal discovery in the presence of latent confounders and selection bias.
\newblock \emph{Artificial Intelligence}, 2008.

\bibitem[Zhu et~al.(2022)Zhu, Liu, Chrysos, and Cevher]{zhu2022robustness}
Z.~Zhu, F.~Liu, G.~Chrysos, and V.~Cevher.
\newblock Robustness in deep learning: The good (width), the bad (depth), and the ugly (initialization).
\newblock In \emph{Advances in neural information processing systems (NeurIPS)}, 2022.

\end{thebibliography}
\bibliographystyle{abbrvnat}

\newpage
\appendix
\onecolumn
\allowdisplaybreaks
\section*{Appendix introduction} 
\label{sec:appendix_intro}
The Appendix is organized as follows:
\begin{itemize}
    \item In ~\cref{sec:symbols_and_notations}, we provide a summary of the symbols and notations used throughout this paper.
    \item In ~\cref{sec:backgrounds}, we provide some background to some of the content covered in this paper.
    \item In~\cref{sec:lemmas}, we present several relevant lemmas that are essential to the proofs in this paper.
    \item In~\cref{sec:proof_of_structural_assumptions}, we provide the proof of~\cref{lemma:structural_assumptions}.
    \item In~\cref{sec:proof_error_bound_sm_causal}, we provide the proof of~\cref{thm:score_bound_causal}.
    \item In~\cref{sec:proof_error_bound_order_causal}, we provide the proof of~\cref{thm:topological_order_bound_causal}.
    \item In~\cref{sec:proof_error_bound_diffusion}, we provide the proof of~\cref{thm:score_bound_diffusion}.
    \item In~\cref{sec:discussion}, we discuss the~\cref{assumption:lipschitz}, the Lipschitz property of score function.
    \item Finally, in~\cref{sec:broader_impacts}, we discuss the broader impacts of this paper.
\end{itemize}

\section{Symbols and Notation}
\label{sec:symbols_and_notations}
In the paper, vectors are indicated with bold small letters, and matrices with bold capital letters. To facilitate the understanding of our work, we include some core symbols and notation in \cref{table:symbols_and_notations}. 

\begin{table}[ht]

\caption{Core symbols and notations used in this project.}
\label{table:symbols_and_notations}
\small
\centering
\begin{tabular}{c | c | c}
\toprule
Symbol & Dimension(s) & Definition \\
\midrule
$\mathcal{S} $ & - & Function space \\
$\mathcal{N}_c(\cdot, \mathcal{S}) $ & $\mathbb{R}$ & Covering number of function space $\mathcal{S}$ \\
$\mathcal{N}(\mu,\sigma^2) $ & - & Gaussian distribution with mean $\mu$ and variance $\sigma^2$ \\
$p$ & - & Probability density function of a probability distribution \\
$\mathbb{E}$ & - & Expected value \\
$[L]$ & - & Shorthand of $\left \{ 1,2,\dots ,L \right \}$\\
$\mathcal{O}$, $o$, $\Omega$ and $\Theta$ & - & Standard Bachmann–Landau order notation\\
\midrule
$n$ & $\mathbb{R}$ & Number of data \\
$d$ & $\mathbb{R}$ & Data dimension (number of variables in the causal model) \\
$L$ & $\mathbb{R}$ & Depth of the neural network \\
$m$ & $\mathbb{R}$ & Width of the neural network\\
$\phi$ & - & The ReLU activation function\\
\midrule
$x^{(i)}$ & $\mathbb{R}$ & The $i$-th element of the vector $\bm{x}$ \\
$\bm{x}_{(i)}$ & $\mathbb{R}^{d}$ & The $i$-th data point \\
$\bm{x}_t$ & $\mathbb{R}^{d}$ & The data point in time $t$ in diffusion model\\
$\bm{W}_1$ & $\mathbb{R}^{m \times d}$ & Weight matrix for the input layer \\
$\bm{W}_l$ & $\mathbb{R}^{m \times m}$ & Weight matrix for the $l$-th hidden layer \\
$\bm{W}_L$ & $\mathbb{R}^{d \times m}$ & Weight matrix for the output layer \\
\midrule
$\epsilon$ & $\mathbb{R}$ & The
noise introduced by denoising score matching\\
$\sigma$ & $\mathbb{R}$ & The standard deviation of Gaussian noise $\epsilon$ \\
\midrule
$\epsilon_i$ & $\mathbb{R}$ & The noise of $i$-th variable of causal model\\
$\sigma_i$ & $\mathbb{R}$ & The standard deviation of Gaussian noise $\epsilon_i$ \\
$f_i$ & - & Non-linear function of $i$-th variable of causal model \\
$\text{PA}_i(\bm{x})$ & - & The set of parents of $x^{(i)}$ in $\bm{x}$\\
$\text{CH}_j(\bm{x})$ & - & The set of children of $x^{(j)}$ in $\bm{x}$ \\
\midrule
\end{tabular}
\end{table}

\section{More backgrounds}
\label{sec:backgrounds}

\subsection{Covering number}

The basic idea of covering number is to approximate a function space with an infinite number of elements by a finite number of elements. It is used to describe how many elements (or subsets) in a given metric space can be "covered" with a finite number of reference elements (or reference subsets) to ensure that the entire space is covered. It is defined as follows:

\begin{definition}
    We assume there exists $m = m(\epsilon)$ elements $f_1, \dots, f_m$ such that for any $f \in \mathcal{F}, \exists i \in \left \{ 1, \dots, m \right \} $ such that $d (f, f_i) \leq \epsilon$. The minimal possible number $m(\epsilon)$ is the covering number of $\mathcal{F}$ at precision $\epsilon$.
\end{definition}

In learning theory, covering number can be used to bound the Rademacher complexity~\citep{shalev2014understanding} then it is related to generalization.

\subsection{More backgrounds about~\cref{alg:algorithm_SCORE}}
The main source of inspiration of the~\citet{rolland2022score} to design~\cref{alg:algorithm_SCORE} is the following lemma:

\begin{lemma} [Adapted from Lemma 1 in~\citet{rolland2022score}]
\label{lemma:var_grad_score}
Let $p$ be the probability density function of a random variable $\bm{x}$ defined via a non-linear additive Gaussian noise model~\cref{eq:causal_model}, and let $\bm{s}(\bm{x}) = \nabla \log p(\bm{x})$ be the associated score function. Then, $\forall j\in [d]$, we have:
\begin{enumerate}
    \item $j$ is a leaf $\Leftrightarrow$ $\forall \bm{x}, \frac{\partial s_j(\bm{x})}{\partial x^{(j)}} = c$, with $c \in \mathbb{R}$ independent of $\bm{x}$, i.e., $\text{Var}\big(\frac{\partial s_j(\bm{x})}{\partial x^{(j)}}\big) = 0$.
    \item $j$ is a leaf, $i$ is a parent of $j$ $\Leftrightarrow$ $s_j(\bm{x})$ depends on $\bm{x}^{(i)}$, i.e., $\text{Var}\big(\frac{\partial s_j(\bm{x})}{\partial x^{(i)}}\big) \neq 0$.
\end{enumerate}
\end{lemma}

\cref{lemma:var_grad_score} reveals the important properties of the nonlinear additive Gaussian noise model: for non-linear additive Gaussian noise models, leaf nodes (and only leaf nodes) have the property that the associated diagonal element in the score’s Jacobian is a constant. Therefore, by repeating this method and always removing the identified leaves, we can estimate a full topological order. This procedure is summarized in~\cref{alg:algorithm_SCORE}.

\section{Relevant Lemmas}
\label{sec:lemmas}

\begin{lemma}[Adapted from Lemma 10 in~\citet{chen2020distribution}]
\label{lemma:approximate_ReLU_network}
For any $\varepsilon \in (0,1)$ and any target $1$-Lipschitz function $\tilde{\bm{s}}$ that defined on $[0,1]^d$ with $\tilde{\bm{s}}(0) = 0$, the architecture yields an approximation $\bm{s} \in \mathcal{S}$ satisfying $\left \| \bm{s}- \tilde{\bm{s}} \right \|_{\infty}  \leq \varepsilon$. 

The configuration of network architecture is:
\begin{equation*}
\begin{split}
& \left \| \bm{s}_l \right \|_{\infty} \leq \sqrt{d}\,, \quad l \in [L] \,,  \\
& \left \| \bm{W}_l \right \|_{\infty} \leq \mathcal{O}(1)\,, \quad l \in [L] \,,  \\
& L = \mathcal{O} (\log \frac{1}{\varepsilon} +d)\,, \\
& m = \mathcal{O} (\frac{1}{\varepsilon^d})\,, \\
& \sum_{l=1}^{L} \left \| \bm{W}_l \right \|_{0} \leq \mathcal{O}(\frac{1}{\varepsilon^d} (\log \frac{1}{\varepsilon} +d))\,.
\end{split}
\end{equation*}
\end{lemma}

\begin{lemma}[Adapted from Theorem 4.4.5 in~\citet{vershynin12}]
\label{lemma:bound_for_Gaussian_matrix_fabinus_norm}
Let $\bm{W}$ be an $N \times n$ matrix whose entries are independent standard normal random variables. Then for every $t \geq 0$, with probability at least $1-2\exp(-t^2/2)$, one has:
\begin{equation*}
s(\bm{A})_{\max} \leq \sqrt{N} + \sqrt{n} +t\,,
\end{equation*}
where the $s(\bm{W})_{\max}$ represent the largest singular value of $\bm{W}$.
\end{lemma}

\begin{lemma}
\label{lemma:bound_data}
If a causal model~\cref{eq:causal_model} satisfies~\cref{assumption:structural_assumptions}. Then with probability at least $1-\frac{1}{n^2 d}$ we have:
\begin{equation*}
    \left \| \bm{x} \right \|_2^2 \leq \sum_{i=1}^{d} (C_i + 2\sigma_i\sqrt{\log nd})^2.
\end{equation*}

\end{lemma}

\begin{proof}
Firstly, we can derive that:

\begin{equation*}
    \left \| \bm{x} \right \|_2^2  = \sum_{i=1}^{d} (x^{(i)})^2= \sum_{i=1}^{d} (f_i + \epsilon_i)^2\leq \sum_{i=1}^{d} (C_i + \left |\epsilon_i\right |)^2 \,.
\end{equation*}

Since $\epsilon_i \sim \mathcal{N}(0, \sigma_i^2)$, according to the tail bound of Gaussian distribution, with probability at least $1-\exp(-\frac{t_i^2}{2\sigma_i^2})$ we have $\left | \epsilon_i \right |  \leq t_i$. Thus:
\begin{equation*}
    \left \| \bm{x} \right \|_2^2 \leq \sum_{i=1}^{d} (C_i + t_i)^2\,,
\end{equation*}

with probability at least $1-\sum_{i=1}^{d}\exp(-\frac{t_i^2}{2\sigma_i^2})$.

Choose $t_i = 2\sigma_i\sqrt{\log nd}$, then we have:

\begin{equation*}
    \left \| \bm{x} \right \|_2^2 \leq \sum_{i=1}^{d} (C_i + 2\sigma_i\sqrt{\log nd})^2,
\end{equation*}

with probability at least $1-\frac{1}{n^2 d}$.

\end{proof}
\section{Proof of~\cref{lemma:structural_assumptions}}
\label{sec:proof_of_structural_assumptions}

\begin{proof}
    According to~\cref{eq:causal_model_pdf}, we can derive that:

    \begin{equation*}
    \begin{split}
        \log p(\bm{x}) & = \sum_{i=1}^{d}\log p(x^{(i)}|\text{PA}_i(\bm{x}))\\
        & = -\frac{1}{2}\sum_{i=1}^{d} \bigg(\frac{x^{(i)} - f_i(\text{PA}_i(\bm{x}))}{\sigma_i} \bigg)^2 -\frac{1}{2}\sum_{i=1}^{d}  \log (2\pi \sigma_i^2)\,.
    \end{split}
    \end{equation*}

    Then:

    \begin{equation}
        s_j(\bm{x}) = \frac{f_j(\text{PA}_j(\bm{x})) - x^{(j)}}{\sigma_j^2} + \sum_{i \in \text{CH}_j(\bm{x})}  \frac{\partial f_i(\text{PA}_i(\bm{x}))}{\partial x^{(j)}}\frac{\epsilon_i}{\sigma_i^2}\,.
    \label{eq:causal_score_function}
    \end{equation}

If $j$ is a leaf:

\begin{equation}
     \frac{\partial s_j(\bm{x})}{\partial x^{(j)}} = -\frac{1}{\sigma_j^2}\,, \quad \text{Var}\bigg(\frac{\partial s_j(\bm{x})}{\partial x^{(j)}}\bigg) = 0\,.
\label{eq:Identifiability_margin_leaf}
\end{equation}

If $j$ is not a leaf:

\begin{equation*}
     \frac{\partial s_j(\bm{x})}{\partial x^{(j)}} = -\frac{1}{\sigma_j^2} + \sum_{i \in \text{CH}_j(\bm{x})}  \frac{\partial^2 f_i(\text{PA}_i(\bm{x}))}{\partial x^{(j)2}}\frac{\epsilon_i}{\sigma_i^2}\,,
\end{equation*}

where the $\text{PA}_i(\bm{x})$ represent the set of parents of $x^{(i)}$ in $\bm{x}$. Then, according to the independence of $\epsilon_i$:

\begin{equation}
\begin{split}
    \text{Var}\bigg(\frac{\partial s_j(\bm{x})}{\partial x^{(j)}}\bigg) & = \sum_{i \in \text{CH}_j(\bm{x})} \text{Var}\bigg(\frac{\partial^2 f_i(\text{PA}_i(\bm{x}))}{\partial x^{(j)2}}\frac{\epsilon_i}{\sigma_i^2}\bigg) \\
    & \geq \text{Var}\bigg(\frac{\partial^2 f_i(\text{PA}_i(\bm{x}))}{\partial x^{(j)2}}\frac{\epsilon_i}{\sigma_i^2}\bigg) \quad \forall i \in \text{CH}_j(\bm{x})\\
    & = \mathbb{E}_{p(\bm{x})}\bigg( \frac{\partial^2 f_i(\text{PA}_i(\bm{x}))}{\partial x^{(j) 2}}^2 \bigg)  \text{Var}\bigg( \frac{\epsilon_i}{\sigma_i^2} \bigg) \quad \forall i \in \text{CH}_j(\bm{x})\\
    & \geq C_m\,.
\end{split}
\label{eq:Identifiability_margin_parent}
\end{equation}

Combine~\cref{eq:Identifiability_margin_leaf,eq:Identifiability_margin_parent}, which concludes the proof.
\end{proof}

\section{Proof of the error bound of score function estimate for the causal model (\cref{thm:score_bound_causal})}
\label{sec:proof_error_bound_sm_causal}
\begin{proof}
Firstly, we use oracle inequality to decompose $J_{\text{DSM}}(\hat{\bm{s}}, p(\bm{x}))$, for any $a \in (0,1)$ and a fixed function $\overline{\bm{s}}$, we have:

\begin{equation*}
\begin{split}
    J_{\text{DSM}}(\hat{\bm{s}}, p(\bm{x})) & = J_{\text{DSM}}(\hat{\bm{s}}, p(\bm{x})) - (1+a)\hat{J}_{\text{DSM}}(\hat{\bm{s}}, p(\bm{x})) + (1+a)\hat{J}_{\text{DSM}}(\hat{\bm{s}}, p(\bm{x}))\\
    & =  J_{\text{DSM}}(\hat{\bm{s}}, p(\bm{x})) - (1+a)\hat{J}_{\text{DSM}}(\hat{\bm{s}}, p(\bm{x})) + (1+a)\inf_{\bm{s}\in \mathcal{S}}\hat{J}_{\text{DSM}}(\bm{s}, p(\bm{x}))\\
    & \leq J_{\text{DSM}}(\hat{\bm{s}}, p(\bm{x})) - (1+a)\hat{J}_{\text{DSM}}(\hat{\bm{s}}, p(\bm{x})) \\
    & + (1+a) \big(\hat{J}_{\text{DSM}}(\overline{\bm{s}}, p(\bm{x})) - (1+a)J_{\text{DSM}}(\overline{\bm{s}}, p(\bm{x}))+(1+a)J_{\text{DSM}}(\overline{\bm{s}}, p(\bm{x}))\big)\\
    & = \bigg(J_{\text{DSM}}(\hat{\bm{s}}, p(\bm{x})) - (1+a)\hat{J}_{\text{DSM}}(\hat{\bm{s}}, p(\bm{x}))\bigg) \\
    & + (1+a)\bigg(\hat{J}_{\text{DSM}}(\overline{\bm{s}}, p(\bm{x})) - (1+a)J_{\text{DSM}}(\overline{\bm{s}}, p(\bm{x}))\bigg) + (1+a)^2 J_{\text{DSM}}(\overline{\bm{s}}, p(\bm{x}))\,.
\end{split}
\end{equation*}

\paragraph{First term}

Firstly, we define that:

\begin{equation*}
    j_{\text{DSM}}(\bm{s}, \bm{x}, p(\bm{x})) = \mathbb{E}_{\hat{\bm{x}}\sim p(\hat{\bm{x}}|\bm{x})} \left \| \bm{s}(\hat{\bm{x}}) - \frac{\partial \log p(\hat{\bm{x}}|\bm{x})}{\partial \hat{\bm{x}}}  \right \|_2^2\,.
\end{equation*}

For any $\bm{s} \in \mathcal{S}$, we have:
\begin{equation}
\begin{split}
    j_{\text{DSM}}(\bm{s}, \bm{x}, p(\bm{x})) & = \mathbb{E}_{\hat{\bm{x}}\sim p(\hat{\bm{x}}|\bm{x})} \left \| \bm{s}(\hat{\bm{x}}) - \frac{\partial \log p(\hat{\bm{x}}|\bm{x})}{\partial \hat{\bm{x}}}  \right \|_2^2\\
    & \leq 2 \mathbb{E}_{\hat{\bm{x}}\sim p(\hat{\bm{x}}|\bm{x})} \bigg(\left \| \bm{s}(\hat{\bm{x}})\right \|_2^2+\left \| \frac{\partial \log p(\hat{\bm{x}}|\bm{x})}{\partial \hat{\bm{x}}}  \right \|_2^2\bigg) \\
    & = 2 \mathbb{E}_{\hat{\bm{x}}\sim p(\hat{\bm{x}}|\bm{x})} \bigg(\left \| \bm{s}(\hat{\bm{x}})\right \|_2^2+\left \| \frac{\bm{x} - \hat{\bm{x}}}{\sigma^2} \right \|_2^2\bigg) \,.
\end{split}
\label{eq:causal_term1_step1}
\end{equation}

For the first part, recall that:
\begin{equation*}
    \hat{\bm{x}} = \bm{x} + \bm{\epsilon}, \ \bm{\epsilon} \sim \mathcal{N}(0, \sigma^2 \bm{I})\,.
\end{equation*}

Then we have:

\begin{equation*}
    \left \| \frac{\hat{\bm{x}} - \bm{x}}{\sigma} \right \|_2^2 \sim \chi^2(d)\,.
\end{equation*}

According to the Bernstein's inequality~\citep{vershynin12} and choose $t = \frac{1}{2}$, we have:

\begin{equation*}
    \mathbb{P}\bigg( \left | \frac{\left \| \frac{\hat{\bm{x}} - \bm{x}}{\sigma} \right \|_2^2}{d}-1 \right | \geq \frac{1}{2} \bigg)\leq 2\exp(-\frac{d}{32})\,.
\end{equation*}

Then we have:

\begin{equation*}
\begin{split}
    \mathbb{P}\bigg( \left \|\hat{\bm{x}} - \bm{x} \right \|_2 \geq \sigma\sqrt{\frac{3d}{2}}\bigg) & = \mathbb{P}\bigg( \left \|\hat{\bm{x}} - \bm{x} \right \|_2^2 \geq \frac{3\sigma^2 d}{2}\bigg) \\
    & = \mathbb{P}\bigg( \frac{\left \|\frac{\hat{\bm{x}} - \bm{x}}{\sigma} \right \|_2^2 }{d} \geq \frac{3 }{2}\bigg) \\
    & \leq \mathbb{P}\bigg( \left | \frac{\left \| \frac{\hat{\bm{x}} - \bm{x}}{\sigma} \right \|_2^2}{d}-1 \right | \geq \frac{1}{2} \bigg)\\
    & \leq 2\exp(-\frac{d}{32})\,.
\end{split}
\end{equation*}

By~\cref{lemma:bound_data}, we have:

\begin{equation*}
\begin{split}
    \left \| \hat{\bm{x}}\right \|_2 & \leq \left \| \hat{\bm{x}} - \bm{x}\right \|_2 + \left \| \bm{x}\right \|_2\\
    & \leq \sigma\sqrt{\frac{3d}{2}} + \sqrt{\sum_{i=1}^{d} (C_i + 2\sigma_i\sqrt{\log nd})^2}\,,
\end{split}
\end{equation*}

with probability at least $1- 2\exp(-\frac{d}{32})-\frac{1}{n^2 d}$ over the randomness of noise $\bm{\epsilon}$ and $\epsilon_i$.

Then by~\cref{lemma:bound_data} and~\citet{Nguyen2021eigenvalue}[Lemma C.1]:

\begin{equation}
    2 \mathbb{E}_{\hat{\bm{x}}\sim p(\hat{\bm{x}}|\bm{x})} \left \| \bm{s}(\hat{\bm{x}})\right \|_2^2 \lesssim \sigma^2 d + \sum_{i=1}^{d} (C_i + 2\sigma_i\sqrt{\log nd})^2\,.
\label{eq:causal_term1_step1_part1}
\end{equation}

with probability at least $1-  2\exp(-\frac{d}{32}) - L\exp(-\Omega(m))-\frac{1}{n^2 d}$ over the randomness of initialization $ \bm{W}$, noise $\bm{\epsilon}$ and $\epsilon_i$.

For the second part:

\begin{equation}
\begin{split}
    2 \mathbb{E}_{\hat{\bm{x}}\sim p(\hat{\bm{x}}|\bm{x})} \left \| \frac{\bm{x} - \hat{\bm{x}}}{\sigma^2} \right \|_2^2 & = 2 \mathbb{E}_{\bm{\epsilon} \sim \mathcal{N}(0, \sigma^2 \bm{I})} \left \| \frac{\bm{\epsilon}}{\sigma^2} \right \|_2^2\\
    & = 2 \mathbb{E}_{\bm{\epsilon}' \sim \mathcal{N}(0, \bm{I})} \left \| \bm{\epsilon}' \right \|_2^2\\
    & = 2 \mathbb{E}_{\epsilon'' \sim \chi^2(d)} \epsilon''\\
    & = 2d\,.
\end{split}
\label{eq:causal_term1_step1_part2}
\end{equation}

Combine~\cref{eq:causal_term1_step1,eq:causal_term1_step1_part1,eq:causal_term1_step1_part2}, we have:

\begin{equation}
\begin{split}
    j_{\text{DSM}}(\bm{s}, \bm{x}, p(\bm{x})) & \leq 2 \mathbb{E}_{\hat{\bm{x}}\sim p(\hat{\bm{x}}|\bm{x})} \bigg(\left \| \bm{s}(\hat{\bm{x}})\right \|_2^2+\left \| \frac{\bm{x} - \hat{\bm{x}}}{\sigma^2} \right \|_2^2\bigg) \\
    & \lesssim (\sigma^2+2)d + \sum_{i=1}^{d} (C_i + 2\sigma_i\sqrt{\log nd})^2\,,
\end{split}
\label{eq:causal_term1_step2}
\end{equation}

with probability at least $1-  2\exp(-\frac{d}{32}) - L\exp(-\Omega(m))-\frac{1}{n^2 d}$ over the randomness of initialization $ \bm{W}$, noise $\bm{\epsilon}$ and $\epsilon_i$.

According to the Bernstein-type concentration inequality~\citet{chen2023score}[Lemma 15], for $\delta \in (0,1)$, $a\leq 1$ and $\tau >0$, we have:

\begin{equation*}
    J_{\text{DSM}}(\hat{\bm{s}}, p(\bm{x})) - (1+a)\hat{J}_{\text{DSM}}(\hat{\bm{s}}, p(\bm{x})) \lesssim \frac{1+3/a}{2n}\big((\sigma^2+2)d + \sum_{i=1}^{d} (C_i + 2\sigma_i\sqrt{\log nd})^2\big)\log\frac{\mathcal{N}_{c}(\tau, \mathcal{S})}{\delta}+(2+a)\tau\,,
\end{equation*}

with probability at least $1- \delta - 2\exp(-\frac{d}{32}) - L\exp(-\Omega(m))-\frac{1}{n d}$ over the randomness of initialization $ \bm{W}$, noise $\bm{\epsilon}$ and $\epsilon_i$.

\paragraph{Second term}

According to the Bernstein-type concentration inequality~\citet{chen2023score}[Lemma 15] and~\cref{eq:causal_term1_step2}, for $\delta \in (0,1)$, $a \leq 1$, $\tau >0$ and a fixed function $\overline{\bm{s}}$, , we have:

\begin{equation*}
\begin{split}
    \hat{J}_{\text{DSM}}(\overline{\bm{s}}, p(\bm{x})) - (1+a)J_{\text{DSM}}(\overline{\bm{s}}, p(\bm{x})) \lesssim \frac{1+3/a}{2n}\big((\sigma^2+2)d + \sum_{i=1}^{d} (C_i + 2\sigma_i\sqrt{\log nd})^2\big)\log\frac{1}{\delta}+(2+a)\tau\,,
\end{split}
\end{equation*}

with probability at least $1- \delta - 2\exp(-\frac{d}{32}) - L\exp(-\Omega(m))-\frac{1}{n d}$ over the randomness of initialization $ \bm{W}$, noise $\bm{\epsilon}$ and $\epsilon_i$.

\paragraph{Third term}

We can derive that:

\begin{equation*}
    J_{\text{DSM}}(\overline{\bm{s}}, p(\bm{x})) = J_{\text{ESM}}(\overline{\bm{s}}, p(\bm{x})) + J_{\text{DSM}}(\overline{\bm{s}}, p(\bm{x})) - J_{\text{ESM}}(\overline{\bm{s}}, p(\bm{x}))\,.
\end{equation*}

According to \cref{lemma:approximate_ReLU_network}, since the error term is invariant with respect to translations on $\nabla \log p(\cdot)$ and the homogeneity of the ReLU neural network, we can omit $\nabla \log p(\bm{0}) = 0$ and rescale bound for the input data, for any $\varepsilon \in (0,1)$, there exists an approximation function $\overline{\bm{s}}$ satisfying $\left \| \nabla \log p(\cdot) - \overline{\bm{s}}(\cdot) \right \|_{\infty}  \leq \varepsilon$, then we have:

\begin{equation*}
    J_{\text{ESM}}(\overline{\bm{s}}, p(\bm{x})) \leq \frac{d\varepsilon^2}{2}\,, 
\end{equation*}

with probability at least $1 -\frac{1}{n d}$ over the randomness of noise $\epsilon_i$ and satisfy the configuration of network architecture in~\cref{lemma:approximate_ReLU_network}.

According to~\citet{vincent2011connection}, we have:

\begin{equation*}
    J_{\text{DSM}}(\overline{\bm{s}}, p(\bm{x})) - J_{\text{ESM}}(\overline{\bm{s}}, p(\bm{x})) = \frac{1}{2}\mathbb{E}_{\bm{x}}\mathbb{E}_{\hat{\bm{x}}\sim \phi (\bm{x}|\bm{x} )}\big[\left \| \nabla_{\hat{\bm{x}}} \log \phi (\bm{x}|\bm{x} )\right \|_2^2 \big] - \frac{1}{2} \left \| \nabla \log p(\cdot)\right \|_{\ell^2(p)}^2 \,.
\end{equation*}

which is an absolute value that does not depend on $\bm{s}$. So we can define that:

\begin{equation*}
    E_1 \coloneqq \frac{1}{2}\mathbb{E}_{\bm{x}}\mathbb{E}_{\hat{\bm{x}}\sim \phi (\bm{x}|\bm{x} )}\big[\left \| \nabla_{\hat{\bm{x}}} \log \phi (\bm{x}|\bm{x} )\right \|_2^2 \big] - \frac{1}{2} \left \| \nabla \log p(\cdot)\right \|_{\ell^2(p)}^2 \,.
\end{equation*}

So if we choose $\overline{\bm{s}}$ is the approximation function that~\cref{lemma:approximate_ReLU_network} provide, then we have:

\begin{equation*}
    J_{\text{DSM}}(\overline{\bm{s}}, p(\bm{x})) \leq \frac{d\varepsilon^2}{2} +E_1 \,.
\end{equation*}

\paragraph{Putting things together}

Combine all three terms, we have:

\begin{equation*}
\begin{split}
    J_{\text{DSM}}(\hat{\bm{s}}, p(\bm{x})) & \leq \bigg(J_{\text{DSM}}(\hat{\bm{s}}, p(\bm{x})) - (1+a)\hat{J}_{\text{DSM}}(\hat{\bm{s}}, p(\bm{x}))\bigg) \\
    & + (1+a)\bigg(\hat{J}_{\text{DSM}}(\overline{\bm{s}}, p(\bm{x})) - (1+a)J_{\text{DSM}}(\overline{\bm{s}}, p(\bm{x}))\bigg) + (1+a)^2 J_{\text{DSM}}(\overline{\bm{s}}, p(\bm{x}))\\
    & \lesssim \bigg(J_{\text{DSM}}(\hat{\bm{s}}, p(\bm{x})) - (1+a)\hat{J}_{\text{DSM}}(\hat{\bm{s}}, p(\bm{x}))\bigg) \\
    & + (1+a)\bigg(\hat{J}_{\text{DSM}}(\overline{\bm{s}}, p(\bm{x})) - (1+a)J_{\text{DSM}}(\overline{\bm{s}}, p(\bm{x}))\bigg) + (1+a)^2 \bigg(\frac{d\varepsilon^2}{2} +E_1\bigg)\\
    & = \bigg(J_{\text{DSM}}(\hat{\bm{s}}, p(\bm{x})) - (1+a)\hat{J}_{\text{DSM}}(\hat{\bm{s}}, p(\bm{x}))\bigg) \\
    & + (1+a)\bigg(\hat{J}_{\text{DSM}}(\overline{\bm{s}}, p(\bm{x})) - (1+a)J_{\text{DSM}}(\overline{\bm{s}}, p(\bm{x}))\bigg) + (1+a)^2 \frac{d\varepsilon^2}{2} +(2a+a^2) E_1 + E_1\,.
\end{split}
\end{equation*}

Then:

\begin{equation*}
\begin{split}
    J_{\text{ESM}}(\hat{\bm{s}}, p(\bm{x})) & = J_{\text{DSM}}(\hat{\bm{s}}, p(\bm{x})) - E_1\\
    & \lesssim \bigg(J_{\text{DSM}}(\hat{\bm{s}}, p(\bm{x})) - (1+a)\hat{J}_{\text{DSM}}(\hat{\bm{s}}, p(\bm{x}))\bigg) \\
    & + (1+a)\bigg(\hat{J}_{\text{DSM}}(\overline{\bm{s}}, p(\bm{x})) - (1+a)J_{\text{DSM}}(\overline{\bm{s}}, p(\bm{x}))\bigg) + (1+a)^2 \frac{d\varepsilon^2}{2} +(2a+a^2) E_1\\
    & \lesssim \frac{1+3/a}{2n}\big((\sigma^2+2)d + \sum_{i=1}^{d} (C_i + 2\sigma_i\sqrt{\log nd})^2\big)\log\frac{\mathcal{N}_{c}(\tau, \mathcal{S})}{\delta}+(2+a)\tau\\
    & +(1+a)\bigg(\frac{1+3/a}{2n}\big((\sigma^2+2)d + \sum_{i=1}^{d} (C_i + 2\sigma_i\sqrt{\log nd})^2\big)\log\frac{1}{\delta}+(2+a)\tau\bigg)\\
    & + (1+a)^2 \frac{d\varepsilon^2}{2} +(2a+a^2) E_1\,,
\end{split}
\end{equation*}

with probability at least $1- 2\delta - 4\exp(-\frac{d}{32}) - 2L\exp(-\Omega(m))-\frac{1}{nd}$ over the randomness of initialization $ \bm{W}$, noise $\bm{\epsilon}$ and $\epsilon_i$.

Let $a = \varepsilon^2$, $\tau = \frac{1}{n}$, $\sigma_i \eqsim \sigma$ and $\frac{C_i}{\sigma_i}\eqsim 1\,,\ \forall i \in [d]$. Then we have:

\begin{equation*}
    J_{\text{ESM}}(\hat{\bm{s}}, p(\bm{x})) \lesssim \frac{\sigma^2 d\log nd}{n\varepsilon^2}\log\frac{\mathcal{N}_{c}(\frac{1}{n}, \mathcal{S})}{\delta}+\frac{1}{n}+d\varepsilon^2\,,
\end{equation*}

with probability at least $1- 2\delta - 4\exp(-\frac{d}{32}) - 2L\exp(-\Omega(m)) - \frac{1}{nd}$ over the randomness of initialization $ \bm{W}$, noise $\bm{\epsilon}$ and $\epsilon_i$.
\end{proof}

\section{Proof of the error bound of topological ordering using the SCORE algorithm in a causal model (\cref{thm:topological_order_bound_causal})}
\label{sec:proof_error_bound_order_causal}

\begin{proof}
We set the weights of the neural network after training are $\widehat{\bm{W}}$. i.e.

\begin{equation*}
    \bm{s}(\bm{x}) = \widehat{\bm{W}}_L\phi(\widehat{\bm{W}}_{L-1}\cdots\phi(\widehat{\bm{W}_1}\bm{x})\cdots)\,.
\end{equation*}

According to the standard chain rule and~\citet{zhu2022robustness}[Lemma 3], we have: 
\begin{equation*}
    \nabla_{\bm{x}}\bm{s}(\bm{x})^{\top}=\widehat{\bm{W}}_L\widehat{\bm{D}}_{L-1}\widehat{\bm{W}}_{L-1}\cdots\widehat{\bm{D}}_1\widehat{\bm{W}}_1\,.
\end{equation*}

Let $\bm{v}_i$ be a one-hot vector with length $d$, with the $i$-th element is $1$ and the rest of the elements are $0$, then we have:

\begin{equation}
\begin{split}
    \frac{\partial s_i(\bm{x})}{\partial x^{(i)}} & = \bm{v}_i \widehat{\bm{W}}_L\widehat{\bm{D}}_{L-1}\widehat{\bm{W}}_{L-1}\cdots\widehat{\bm{D}}_1\widehat{\bm{W}}_1 \bm{v}_i\\
    & \leq \left \| \bm{v}_i \right \|_2 \left \| \widehat{\bm{W}}_L\widehat{\bm{D}}_{L-1}\widehat{\bm{W}}_{L-1}\cdots\widehat{\bm{D}}_1 \right \|_2 \left \| \widehat{\bm{W}}_1 \right \|_2 \left \| \bm{v}_i \right \|_2 \\
    & = \left \| \widehat{\bm{W}}_L\widehat{\bm{D}}_{L-1}\widehat{\bm{W}}_{L-1}\cdots\widehat{\bm{D}}_1 \right \|_2 \left \| \widehat{\bm{W}}_1\right \|_2\\
    & = \bigg(\left \| \bm{W}_L\bm{D}_{L-1}\bm{W}_{L-1}\cdots\bm{D}_1 \right \|_2+\left \| \widehat{\bm{W}}_L\widehat{\bm{D}}_{L-1}\widehat{\bm{W}}_{L-1}\cdots\widehat{\bm{D}}_1 -\bm{W}_L\bm{D}_{L-1}\bm{W}_{L-1}\cdots\bm{D}_1 \right \|_2\bigg)\\
    & \times \bigg(\left \| \bm{W}_1\right \|_2 + \left \| \widehat{\bm{W}}_1 - \bm{W}_1 \right \|_2\bigg)\\
    & \coloneqq (T_1+T_2)\times(T_3+T_4)\,.
\end{split}
\label{eq:causal}
\end{equation}

Firstly, we focus on $T_1$. Define $\bm{t}_{l}(\bm{v}) = \bm{D}_{l}\bm{W}_l \cdots\bm{D}_1\bm{v}$, then for any vector $\bm{v}$ that satisfy $\left \| \bm{v} \right \|_2 = 1$:

\begin{equation}
\begin{split}
    \left \| \bm{W}_L\bm{D}_{L-1}\bm{W}_{L-1}\cdots\bm{D}_1 \bm{v}\right \|_2 & = \left \| \bm{W}_L \bm{t}_{L-1}(\bm{v})\right \|_2 \\
    & = \sqrt{\left \| \bm{W}_L \bm{t}_{L-1}(\bm{v})\right \|_2^2}\\
    & = \sqrt{\frac{\left \| \bm{W}_L \bm{t}_{L-1}(\bm{v})\right \|_2^2}{\left \| \bm{t}_{L-1}(\bm{v})\right \|_2^2}\frac{\left \| \bm{t}_{L-1}(\bm{v})\right \|_2^2}{\left \| \bm{t}_{L-2}(\bm{v})\right \|_2^2}\cdots\frac{\left \| \bm{t}_{2}(\bm{v})\right \|_2^2}{\left \| \bm{t}_{1}(\bm{v})\right \|_2^2}\left \| \bm{t}_{1}(\bm{v})\right \|_2^2}\,.
\end{split}
\label{eq:causal_chi_square}
\end{equation}

According to~\citet{zhu2022robustness}[Lemma 2], we have:
\begin{equation*}
    \frac{\left \| \bm{t}_{l}(\bm{v})\right \|_2^2}{\left \| \bm{t}_{l-1}(\bm{v})\right \|_2^2} \sim \frac{2}{m}\chi^2(\varrho ),\quad \forall l = 2, \cdots ,L-1 \,,
\end{equation*}

where $\varrho \sim \mathrm{Ber}(m,1/2)$.

According to~\citet{Ghosh2021Chisquared}, with probability at least $1-\exp(-\Theta(m))$ over the randomness of initialization $\bm{W}_l$, we have:

\begin{equation}
    \frac{\left \| \bm{t}_{l}(\bm{v})\right \|_2^2}{\left \| \bm{t}_{l-1}(\bm{v})\right \|_2^2} \leq 4,\quad \forall l = 2, \cdots ,L-1 \,.
\label{eq:causal_chi_square_1}
\end{equation}

By the definition of chi-square distribution, we have:

\begin{equation*}
    \frac{\left \| \bm{W}_{L}\bm{t}_{L-1} \right \|_2^2 }{\left \| \bm{t}_{L-1} \right \|_2^2}\sim \frac{\chi^2(d)}{d}\,,
\end{equation*}

Similar, according to~\citet{Ghosh2021Chisquared}, with probability at least $1-\exp(-\Theta(d))$ over the randomness of initialization $\bm{W}_L$, we have:

\begin{equation}
    \frac{\left \| \bm{W}_{L}\bm{t}_{L-1} \right \|_2^2 }{\left \| \bm{t}_{L-1} \right \|_2^2} \leq 2\,.
\label{eq:causal_chi_square_2}
\end{equation}

And we can derive that:
\begin{equation}
    \left \| \bm{t}_{1}(\bm{v})\right \|_2^2 = \left \| \bm{D}_1\bm{v} \right \|_2^2 \leq \bigg( \left \| \bm{D}_1 \right \|_2 \left \| \bm{v} \right \|_2\bigg)^2 \leq 1\,.
\label{eq:causal_chi_square_3}
\end{equation}

Combine~\cref{eq:causal_chi_square,eq:causal_chi_square_1,eq:causal_chi_square_2,eq:causal_chi_square_3}, we have:

\begin{equation*}
    \left \| \bm{W}_L\bm{D}_{L-1}\bm{W}_{L-1}\cdots\bm{D}_1 \bm{v}\right \|_2 = \sqrt{\frac{\left \| \bm{W}_L \bm{t}_{L-1}(\bm{v})\right \|_2^2}{\left \| \bm{t}_{L-1}(\bm{v})\right \|_2^2}\frac{\left \| \bm{t}_{L-1}(\bm{v})\right \|_2^2}{\left \| \bm{t}_{L-2}(\bm{v})\right \|_2^2}\cdots\frac{\left \| \bm{t}_{2}(\bm{v})\right \|_2^2}{\left \| \bm{t}_{1}(\bm{v})\right \|_2^2}\left \| \bm{t}_{1}(\bm{v})\right \|_2^2} \leq 2^{\frac{2L-1}{2}}\,,
\end{equation*}

with probability at least $1-\exp(-\Theta(d))-(L-2)\exp(-\Theta(m))$ over the randomness of initialization $\bm{W}$.

i.e.

\begin{equation}
    T_1 = \left \| \bm{W}_L\bm{D}_{L-1}\bm{W}_{L-1}\cdots\bm{D}_1\right \|_2 \leq 2^{\frac{2L-1}{2}}\,,
\label{eq:causal_term1}
\end{equation}

with probability at least $1-\exp(-\Theta(d))-(L-2)\exp(-\Theta(m))$ over the randomness of initialization $\bm{W}$.

For a perturbation matrices satisfy $T_4 = \left \| \widehat{\bm{W}}_l - \bm{W}_l \right \|_2 \leq \omega = \mathcal{O}(\frac{1}{L^{3/2}})\,, \ \forall l \in [L]$, by~\citet[Lemma 8.7]{allen2019convergence}, we obtain that for any integer $s = \mathcal{O}(m\omega^{2/3}L)$ and $d \leq \mathcal{O}(\frac{m}{L\log{m}})$, with probability at least $1-\exp{(-\Omega(m\log m \omega^{2/3}L ))}$ over the randomness of initialization $\bm W$, it holds that:

\begin{equation}
    T_2 = \left \| \widehat{\bm{W}}_L\widehat{\bm{D}}_{L-1}\widehat{\bm{W}}_{L-1}\cdots\widehat{\bm{D}}_1 -\bm{W}_L\bm{D}_{L-1}\bm{W}_{L-1}\cdots\bm{D}_1 \right \|_2 \leq \mathcal{O}\bigg(\frac{\omega^{1/3}L^2\sqrt{m\log m}}{\sqrt{d}}\bigg)\,.
\label{eq:causal_term2}
\end{equation}

For $T_3$, according to~\cref{lemma:bound_for_Gaussian_matrix_fabinus_norm}, we have that for every $t \geq 0$, with probability at least $1-2\exp(-t^2/2)$ over the randomness of initialization $\bm{W}_1$, one has:
\begin{equation}
    T_3 = \left \| \bm{W}_1\right \|_2 \leq \sqrt{\frac{2}{m}}(\sqrt{m} + \sqrt{d} +t)\,.
\label{eq:causal_term3}
\end{equation}

Combine~\cref{eq:causal,eq:causal_term1,eq:causal_term2,eq:causal_term3}, choose $t = \sqrt{m}$ we have:

\begin{equation}
\begin{split}
    \frac{\partial s_i(\bm{x})}{\partial x^{(i)}} & \leq (T_1+T_2)\times(T_3+T_4) \\
    & \lesssim   \bigg(2^{\frac{2L-1}{2}} + \frac{\omega^{1/3}L^2\sqrt{m\log m}}{\sqrt{d}}\bigg)\times \bigg(\frac{1}{L^{3/2}}+\sqrt{\frac{2}{m}}(2\sqrt{m} + \sqrt{d})\bigg)\\
    & \lesssim \frac{2^L\sqrt{\log m}(\sqrt{m} + \sqrt{d})}{\sqrt{d}}\,,
\end{split}
\label{eq:causal_2}
\end{equation}

with probability at least $1-\exp(-\Theta(d))-L\exp(-\Theta(m)) -\exp{(-\Omega(m\log m))}$ over the randomness of initialization $\bm W$.

Then, for $\big(\frac{\partial s_i(\bm{x})}{\partial x^{(i)}}\big)$, we have that:

\begin{equation}
    \bigg(\frac{\partial s_i(\bm{x})}{\partial x^{(i)}}\bigg)^2 \lesssim \frac{2^{2L}\log m(m + d)}{d}\,,
\label{eq:causal_3}
\end{equation}
with probability at least $1-\exp(-\Theta(d))-L\exp(-\Theta(m)) -\exp{(-\Omega(m\log m))}$ over the randomness of initialization $\bm W$.

According to Hoeffding's inequality for bounded random variables~\citep{vershynin12}[Thmorem 2.2.6], we have that:

\begin{equation*}
    \left | \frac{1}{n} \sum_{i=1}^{n}\frac{\partial s_i(\bm{x})}{\partial x^{(i)}} - \mathbb{E} \frac{\partial s_i(\bm{x})}{\partial x^{(i)}} \right|\leq \frac{C_m}{12\mathbb{E} \frac{\partial s_i(\bm{x})}{\partial x^{(i)}}}\,,
\end{equation*}

with probability at least $1-\exp(-\Theta(d))-L\exp(-\Theta(m)) -\exp{(-\Omega(m\log m))} - 2\exp(-\Omega(\frac{nC_m^2 d^2}{2^{4L+5}(\log m)^2 (m^2+d^2)} ))$, and

\begin{equation*}
    \left | \frac{1}{n} \sum_{i=1}^{n} \bigg(\frac{\partial s_i(\bm{x})}{\partial x^{(i)}}\bigg)^2 - \mathbb{E} \bigg(\frac{\partial s_i(\bm{x})}{\partial x^{(i)}}\bigg)^2 \right|\leq \frac{C_m}{4}\,,
\end{equation*}

with probability at least $1-\exp(-\Theta(d))-L\exp(-\Theta(m)) -\exp{(-\Omega(m\log m))} - 2\exp(-\frac{nC_m^2 d^2}{2^{4L+5}(\log m)^2 (m^2+d^2)} )$.

Then we have:

\begin{equation*}
\begin{split}
    \left | \text{Var}\bigg(\frac{\partial s_i(\bm{x})}{\partial x^{(i)}}\bigg) - \hat{\text{Var}}\bigg(\frac{\partial s_i(\bm{x})}{\partial x^{(i)}}\bigg)\right | & = \left | \mathbb{E} \bigg(\frac{\partial s_i(\bm{x})}{\partial x^{(i)}}\bigg)^2 - \bigg(\mathbb{E} \frac{\partial s_i(\bm{x})}{\partial x^{(i)}}\bigg)^2 - \sum_{i=1}^{n} \bigg(\frac{\partial s_i(\bm{x})}{\partial x^{(i)}}\bigg)^2 + \bigg(\frac{1}{n} \sum_{i=1}^{n}\frac{\partial s_i(\bm{x})}{\partial x^{(i)}}\bigg)^2 \right |\\
    & \leq \left | \mathbb{E} \bigg(\frac{\partial s_i(\bm{x})}{\partial x^{(i)}}\bigg)^2 -\sum_{i=1}^{n} \bigg(\frac{\partial s_i(\bm{x})}{\partial x^{(i)}}\bigg)^2\right |+ \left |  - \bigg(\mathbb{E} \frac{\partial s_i(\bm{x})}{\partial x^{(i)}}\bigg)^2+ \bigg(\frac{1}{n} \sum_{i=1}^{n}\frac{\partial s_i(\bm{x})}{\partial x^{(i)}}\bigg)^2 \right |\\
    & \leq \frac{C_m}{4} + \left |\frac{1}{n} \sum_{i=1}^{n}\frac{\partial s_i(\bm{x})}{\partial x^{(i)}} - \mathbb{E} \frac{\partial s_i(\bm{x})}{\partial x^{(i)}} \right |\left | \mathbb{E} \frac{\partial s_i(\bm{x})}{\partial x^{(i)}}+ \frac{1}{n} \sum_{i=1}^{n}\frac{\partial s_i(\bm{x})}{\partial x^{(i)}} \right |\\
    & \leq \frac{C_m}{2}\,,
\end{split}
\end{equation*}
\end{proof}

with probability at least $1-\exp(-\Theta(d))-L\exp(-\Theta(m)) -\exp{(-\Theta(m\log m))} - 2\exp(-\frac{nC_m^2 d^2}{2^{4L+5}(\log m)^2 (m^2+d^2)} )$.

Thus, for $i$ is a leaf and $j$ is not a leaf, according to~\cref{assumption:structural_assumptions} and~\cref{lemma:structural_assumptions}, we have:

\begin{equation*}
    \text{Var}\bigg(\frac{\partial s_j(\bm{x})}{\partial x^{(j)}}\bigg) - \text{Var}\bigg(\frac{\partial s_i(\bm{x})}{\partial x^{(i)}}\bigg) \geq C_m\,.
\end{equation*}

Then:

\begin{equation*}
\begin{split}
    \hat{\text{Var}}\bigg(\frac{\partial s_i(\bm{x})}{\partial x^{(i)}}\bigg) & = \hat{\text{Var}}\bigg(\frac{\partial s_i(\bm{x})}{\partial x^{(i)}}\bigg) - \text{Var}\bigg(\frac{\partial s_i(\bm{x})}{\partial x^{(i)}}\bigg) + \text{Var}\bigg(\frac{\partial s_i(\bm{x})}{\partial x^{(i)}}\bigg) \\
    & \leq \frac{C_m}{2} + \text{Var}\bigg(\frac{\partial s_i(\bm{x})}{\partial x^{(i)}}\bigg)\\
    & \leq \text{Var}\bigg(\frac{\partial s_j(\bm{x})}{\partial x^{(j)}}\bigg) - \frac{C_m}{2}\\
    & = \text{Var}\bigg(\frac{\partial s_j(\bm{x})}{\partial x^{(j)}}\bigg) - \hat{\text{Var}}\bigg(\frac{\partial s_j(\bm{x})}{\partial x^{(j)}}\bigg) + \hat{\text{Var}}\bigg(\frac{\partial s_j(\bm{x})}{\partial x^{(j)}}\bigg) -\frac{C_m}{2}\\
    & \leq \frac{C_m}{2} + \hat{\text{Var}}\bigg(\frac{\partial s_j(\bm{x})}{\partial x^{(j)}}\bigg) -\frac{C_m}{2}\\
    & = \hat{\text{Var}}\bigg(\frac{\partial s_j(\bm{x})}{\partial x^{(j)}}\bigg)\,.
\end{split}
\end{equation*}

with probability at least $1-\exp(-\Theta(d))-L\exp(-\Theta(m)) -\exp{(-\Theta(m\log m))} - 2\exp(-\frac{nC_m^2 d^2}{2^{4L+5}(\log m)^2 (m^2+d^2)} )$. 
Considering all variables, then with probability at least:

\begin{equation*}
    1-\exp(-\Theta(d))-(L+1)\exp(-\Theta(m))- 2n\exp(-\frac{nC_m^2 d^2}{2^{4L+5}(\log m)^2 (m^2+d^2)} )\,,
\end{equation*}

that~\cref{alg:algorithm_SCORE} can completely recover the correct topological order of the non-linear additive Gaussian noise model.

\section{Proof of the error bound of score function estimate for the score-based generative modeling (\cref{thm:score_bound_diffusion})}
\label{sec:proof_error_bound_diffusion}

\begin{proof}

Firstly, we use oracle inequality to decompose $\mathcal{L}(\hat{\bm{s}})$, for any $a \in (0,1)$ and a fixed function $\overline{\bm{s}}$, we have:

\begin{equation*}
\begin{split}
    \mathcal{L}(\hat{\bm{s}}) & = \mathcal{L}(\hat{\bm{s}}) - (1+a)\hat{\mathcal{L}}(\hat{\bm{s}}) + (1+a)\hat{\mathcal{L}}(\hat{\bm{s}})\\
    & =  \mathcal{L}(\hat{\bm{s}}) - (1+a)\hat{\mathcal{L}}(\hat{\bm{s}}) + (1+a)\inf_{\bm{s}\in \mathcal{S}}\hat{\mathcal{L}}(\bm{s}) \\
    &\leq \mathcal{L}(\hat{\bm{s}}) - (1+a)\hat{\mathcal{L}}(\hat{\bm{s}}) + (1+a) \big(\hat{\mathcal{L}}(\overline{\bm{s}}) - (1+a)\mathcal{L}(\overline{\bm{s}})+(1+a)\mathcal{L}(\overline{\bm{s}})\big)\\
    & = \bigg(\mathcal{L}(\hat{\bm{s}}) - (1+a)\hat{\mathcal{L}}(\hat{\bm{s}})\bigg) + (1+a)\bigg(\hat{\mathcal{L}}(\overline{\bm{s}}) - (1+a)\mathcal{L}(\overline{\bm{s}})\bigg) + (1+a)^2 \mathcal{L}(\overline{\bm{s}})\,.
\end{split}
\end{equation*}

\paragraph{First term}

For any $\bm{s} \in \mathcal{S}$, we have:
\begin{equation}
\begin{split}
    \ell(\bm{x};\bm{s}) & = \frac{1}{T-t_0}\int_{t_0}^{T} \mathbb{E}_{\bm{x}_t\sim p_{0t} (\bm{x}_t|\bm{x}_0 = \bm{x})}\big[\left \| \nabla_{\bm{x}_t} \log p_{0t} (\bm{x}_t|\bm{x}_0 = \bm{x}) - \bm{s}(\bm{x}_t,t)\right \|_2^2 \big]\mathrm{d}t\\
    & = \frac{1}{T-t_0}\int_{t_0}^{T} \mathbb{E}_{\bm{x}_t\sim p_{0t} (\bm{x}_t|\bm{x}_0 = \bm{x})}\bigg(\left \| \frac{\bm{x}_t-\alpha(t)\bm{x}}{h(t)} + \bm{s}(\bm{x}_t,t)\right \|_2^2 \bigg)\mathrm{d}t\\
    & \leq \frac{3}{T-t_0}\int_{t_0}^{T} \mathbb{E}_{\bm{x}_t\sim p_{0t} (\bm{x}_t|\bm{x}_0 = \bm{x})}\bigg[\bigg(\left \| \frac{\bm{x}_t}{h(t)}\right \|_2^2 + \left \| \frac{\alpha(t)\bm{x}}{h(t)}\right \|_2^2 + \left \| \bm{s}(\bm{x}_t,t)\right \|_2^2 \bigg) \bigg]\mathrm{d}t\\
    & = \frac{3}{T-t_0}\int_{t_0}^{T} \mathbb{E}_{\bm{x}_t\sim p_{0t} (\bm{x}_t|\bm{x}_0 = \bm{x})}\bigg( \left \| \frac{\bm{x}_t}{h(t)}\right \|_2^2 \bigg)\mathrm{d}t\\
    & + \frac{3}{T-t_0}\int_{t_0}^{T} \bigg(\left \| \frac{\alpha(t)\bm{x}}{h(t)}\right \|_2^2 \bigg)\mathrm{d}t\\
    & + \frac{3}{T-t_0}\int_{t_0}^{T} \mathbb{E}_{\bm{x}_t\sim p_{0t} (\bm{x}_t|\bm{x}_0 = \bm{x})}\bigg(\left \| \bm{s}(\bm{x}_t,t)\right \|_2^2 \bigg)\mathrm{d}t\,.
\end{split}
\label{eq:bounded_data_term1_step1}
\end{equation}

For the first part, for forward process SDE~\cref{eq:forward_sde} we can easily derive that $p_{0t} (\bm{x}_t|\bm{x}_0) \sim \mathcal{N}\big(\alpha(t)\bm{x}_0, h(t)I_d\big)$, where $\alpha(t) = e^{-\frac{t}{2}}$ and $h(t) = 1- \alpha(t)^2$.

\begin{equation}
\begin{split}
    & \int_{t_0}^{T} \mathbb{E}_{\bm{x}_t\sim p_{0t} (\bm{x}_t|\bm{x}_0 = \bm{x})}\bigg( \left \| \frac{\bm{x}_t}{h(t)}\right \|_2^2 \bigg)\mathrm{d}t \\
    = & \int_{t_0}^{T} \mathbb{E}_{\bm{x}_t\sim \mathcal{N}\big(\alpha(t)\bm{x}_0, h(t)I_d\big)}\bigg( \left \| \frac{\bm{x}_t}{h(t)}\right \|_2^2 \bigg)\mathrm{d}t \\
    = & \int_{t_0}^{T} \bigg(\sum_{i=1}^{d}\big[\mathbb{E}_{x_t^{(i)}\sim \mathcal{N}\big(\alpha(t)x_0^{(i)}, h(t)\big)}  \bigg(\frac{x_t^{(i)}}{h(t)}\bigg)^2 \big] \bigg)\mathrm{d}t \\
    = & \int_{t_0}^{T}\bigg(\sum_{i=1}^{d}\big[\frac{\alpha(t)^2}{h(t)^2}(x_0^{(i)})^2 + \frac{1}{h(t)}\big] \bigg)\mathrm{d}t \\
    = & \sum_{i=1}^{d}(x_0^{(i)})^2 \int_{t_0}^{T} \frac{\alpha(t)^2}{h(t)^2}\mathrm{d}t + \int_{t_0}^{T} \frac{d}{h(t)}\mathrm{d}t\\
    \leq & \frac{T-t_0}{Tt_0} C_d^2 + d(T-\log(t_0))\,.
\end{split}
\label{eq:bounded_data_term1_step1_part1}
\end{equation}

For the second part:

\begin{equation}
    \int_{t_0}^{T} \left \| \frac{\alpha(t)\bm{x}}{h(t)}\right \|_2^2 \mathrm{d}t\leq C_d^2 \int_{t_0}^{T} \frac{\alpha(t)^2}{h(t)^2} \mathrm{d}t = C_d^2 \frac{T-t_0}{Tt_0}\,.
\label{eq:bounded_data_term1_step1_part2}
\end{equation}

For the third part, by~\citet{Nguyen2021eigenvalue}[Lemma C.1]:

\begin{equation}
    \int_{t_0}^{T} \mathbb{E}_{\bm{x}_t\sim p_{0t} (\bm{x}_t|\bm{x}_0 = \bm{x})}\bigg(\left \| \bm{s}(\bm{x}_t,t)\right \|_2^2 \bigg)\mathrm{d}t \eqsim C_d^2(T-t_0)\,,
\label{eq:bounded_data_term1_step1_part3}
\end{equation}

with probability at least $1- L\exp(-\Omega(m))$ over the randomness of initialization $ \bm{W}$.

Combine~\cref{eq:bounded_data_term1_step1,eq:bounded_data_term1_step1_part1,eq:bounded_data_term1_step1_part2,eq:bounded_data_term1_step1_part3}, we have:

\begin{equation}
\begin{split}
    \ell(\bm{x};\bm{s}) & = \frac{3}{T-t_0}\int_{t_0}^{T} \mathbb{E}_{\bm{x}_t\sim p_{0t} (\bm{x}_t|\bm{x}_0 = \bm{x})}\bigg( \left \| \frac{\bm{x}_t}{h(t)}\right \|_2^2 \bigg)\mathrm{d}t\\
    & + \frac{3}{T-t_0}\int_{t_0}^{T} \bigg(\left \| \frac{\alpha(t)\bm{x}}{h(t)}\right \|_2^2 \bigg)\mathrm{d}t\\
    & + \frac{3}{T-t_0}\int_{t_0}^{T} \mathbb{E}_{\bm{x}_t\sim p_{0t} (\bm{x}_t|\bm{x}_0 = \bm{x})}\bigg(\left \| \bm{s}(\bm{x}_t,t)\right \|_2^2 \bigg)\mathrm{d}t\\
    & \lesssim \frac{3}{T-t_0} \bigg(\frac{T-t_0}{Tt_0}C_d^2 + d(T-\log(t_0))+C_d^2\frac{T-t_0}{Tt_0}+C_d^2(T-t_0)\bigg)\\
    & = \frac{3d(T-\log(t_0))}{T-t_0}+3C_d^2+\frac{6C_d^2}{Tt_0}\,,
\end{split}
\label{eq:bounded_data_term1_step2}
\end{equation}

with probability at least $1- L\exp(-\Omega(m))$ over the randomness of initialization $ \bm{W}$.

According to the Bernstein-type concentration inequality~\citet{chen2023score}[Lemma 15], for $\delta \in (0,1)$, $a \leq 1$ and $\tau >0$, we have:

\begin{equation*}
\begin{split}
    \mathcal{L}(\hat{\bm{s}}) - (1+a)\hat{\mathcal{L}}(\hat{\bm{s}}) \lesssim \frac{1+3/a}{n}\bigg(\frac{d(T-\log(t_0))}{T-t_0}+C_d^2\bigg)\log\frac{\mathcal{N}_{c}(\tau, \mathcal{S})}{\delta}+(2+a)\tau\,,
\end{split}
\end{equation*}

with probability at least $1- \delta- L\exp(-\Omega(m))$ over the randomness of initialization $ \bm{W}$.

\paragraph{Second term}

According to the Bernstein-type concentration inequality~\citet{chen2023score}[Lemma 15] and~\cref{eq:bounded_data_term1_step2}, for $\delta \in (0,1)$, $\tau >0$ and a fixed function $\overline{\bm{s}}$, , we have:

\begin{equation*}
\begin{split}
    \hat{\mathcal{L}}(\overline{\bm{s}}) - (1+a)\mathcal{L}(\overline{\bm{s}}) \lesssim \frac{1+3/a}{n}\bigg(\frac{d(T-\log(t_0))}{T-t_0}+C_d^2\bigg)\log\frac{1}{\delta}+(2+a)\tau\,,
\end{split}
\end{equation*}

with probability at least $1-\delta- L\exp(-\Omega(m))$ over the randomness of initialization $ \bm{W}$.

\paragraph{Third term}

We can derive that:

\begin{equation*}
\begin{split}
    \mathcal{L}(\overline{\bm{s}}) & = \frac{1}{T-t_0}\int_{t_0}^{T} \left \| \nabla \log p_t(\cdot) - \overline{\bm{s}}(\cdot,t)\right \|_{\ell^2(p_t)}^2 \mathrm{d}t \\
    & + \mathcal{L}(\overline{\bm{s}}) - \frac{1}{T-t_0}\int_{t_0}^{T} \left \| \nabla \log p_t(\cdot) - \overline{\bm{s}}(\cdot,t)\right \|_{\ell^2(p_t)}^2 \mathrm{d}t \\
\end{split}
\end{equation*}

For the first part, according to \cref{lemma:approximate_ReLU_network}, since the error term is invariant with respect to translations on $\nabla \log p_t(\cdot)$ and the homogeneity of the ReLU neural network, we can omit $\nabla \log p_t(\bm{0}) = 0$ and rescale bound for the input data, for any $\varepsilon \in (0,1)$, there exist an approximation function $\overline{\bm{s}}$ satisfying $\left \| \nabla \log p_t(\cdot) - \overline{\bm{s}}(\cdot,t) \right \|_{\infty}  \leq \varepsilon$, then we have:

\begin{equation*}
    \frac{1}{T-t_0}\int_{t_0}^{T} \left \| \nabla \log p_t(\cdot) - \overline{\bm{s}}(\cdot,t)\right \|_{\ell^2(p_t)}^2 \mathrm{d}t \leq d\varepsilon^2\,,
\end{equation*}

that satisfy the configuration of network architecture in~\cref{lemma:approximate_ReLU_network}.

For the second part:
\begin{equation*}
\begin{split}
    & \mathcal{L}(\overline{\bm{s}}) - \frac{1}{T-t_0}\int_{t_0}^{T} \left \| \nabla \log p_t(\cdot) - \overline{\bm{s}}(\cdot,t)\right \|_{\ell^2(p_t)}^2 \mathrm{d}t\\
    & = \frac{1}{T-t_0}\int_{t_0}^{T} \bigg(\mathbb{E}_{\bm{x}_0\sim p_0}\mathbb{E}_{\bm{x}_t\sim p_{0t} (\bm{x}_t|\bm{x}_0 )}\big[\left \| \nabla_{\bm{x}_t} \log p_{0t} (\bm{x}_t|\bm{x}_0 ) - \bm{s}(\bm{x}_t,t)\right \|_2^2 \big]- \left \| \nabla \log p_t(\cdot) - \overline{\bm{s}}(\cdot,t)\right \|_{\ell^2(p_t)}^2 \bigg)\mathrm{d}t\,.
\end{split}
\end{equation*}

According to~\citet{vincent2011connection}, we have:

\begin{equation*}
\begin{split}
    & \mathbb{E}_{\bm{x}_0\sim p_0}\mathbb{E}_{\bm{x}_t\sim p_{0t} (\bm{x}_t|\bm{x}_0 = \bm{x}_{(i)})}\big[\left \| \nabla_{\bm{x}_t} \log p_{0t} (\bm{x}_t|\bm{x}_0 = \bm{x}_{(i)}) - \bm{s}(\bm{x}_t,t)\right \|_2^2 \big]- \left \| \nabla \log p_t(\cdot) - \overline{\bm{s}}(\cdot,t)\right \|_{\ell^2(p_t)}^2\\
    = & \mathbb{E}_{\bm{x}_0\sim p_0}\mathbb{E}_{\bm{x}_t\sim p_{0t} (\bm{x}_t|\bm{x}_0 )}\big[\left \| \nabla_{\bm{x}_t} \log p_{0t} (\bm{x}_t|\bm{x}_0 )\right \|_2^2 \big] - \left \| \nabla \log p_t(\cdot)\right \|_{\ell^2(p_t)}^2\,,
\end{split}
\end{equation*}

which is an absolute value that does not depend on $\bm{s}$. So we can define that:

\begin{equation*}
    E_2 \coloneqq \mathbb{E}_{\bm{x}_0\sim p_0}\mathbb{E}_{\bm{x}_t\sim p_{0t} (\bm{x}_t|\bm{x}_0 )}\big[\left \| \nabla_{\bm{x}_t} \log p_{0t} (\bm{x}_t|\bm{x}_0 )\right \|_2^2 \big] - \left \| \nabla \log p_t(\cdot)\right \|_{\ell^2(p_t)}^2\,.
\end{equation*}

So if we choose $\overline{\bm{s}}$ is the approximation function that~\cref{lemma:approximate_ReLU_network} provide, then we have:

\begin{equation*}
    \mathcal{L}(\overline{\bm{s}}) \leq d\varepsilon^2 +E_2 \,.
\end{equation*}

\paragraph{Putting things together}

Combine all three terms, we have:

\begin{equation*}
\begin{split}
    \mathcal{L}(\hat{\bm{s}}) & \leq \bigg(\mathcal{L}(\hat{\bm{s}}) - (1+a)\hat{\mathcal{L}}(\hat{\bm{s}})\bigg) + (1+a)\bigg(\hat{\mathcal{L}}(\overline{\bm{s}}) - (1+a)\mathcal{L}(\overline{\bm{s}})\bigg) + (1+a)^2 \mathcal{L}(\overline{\bm{s}})\\
    & \leq \bigg(\mathcal{L}(\hat{\bm{s}}) - (1+a)\hat{\mathcal{L}}(\hat{\bm{s}})\bigg) + (1+a)\bigg(\hat{\mathcal{L}}(\overline{\bm{s}}) - (1+a)\mathcal{L}(\overline{\bm{s}})\bigg) + (1+a)^2 (d\varepsilon^2 +E_2)\\
    & = \bigg(\mathcal{L}(\hat{\bm{s}}) - (1+a)\hat{\mathcal{L}}(\hat{\bm{s}})\bigg) + (1+a)\bigg(\hat{\mathcal{L}}(\overline{\bm{s}}) - (1+a)\mathcal{L}(\overline{\bm{s}})\bigg) + (1+a)^2 d\varepsilon^2 + (2a+a^2)E_2 + E_2\\
\end{split}
\end{equation*}

Then:

\begin{equation*}
\begin{split}
    & \frac{1}{T-t_0}\int_{t_0}^{T} \left \| \nabla \log p_t(\cdot) - \hat{\bm{s}}(\cdot,t)\right \|_{\ell^2(p_t)}^2 \mathrm{d}t \\
    = & \mathcal{L}(\hat{\bm{s}}) - E_2\\
    = & \bigg(\mathcal{L}(\hat{\bm{s}}) - (1+a)\hat{\mathcal{L}}(\hat{\bm{s}})\bigg) + (1+a)\bigg(\hat{\mathcal{L}}(\overline{\bm{s}}) - (1+a)\mathcal{L}(\overline{\bm{s}})\bigg) + (1+a)^2 d\varepsilon^2 + (2a+a^2)E_2\\
    \lesssim & \bigg(\frac{1+3/a}{n}\bigg(\frac{d(T-\log(t_0))}{T-t_0}+C_d^2\bigg)\log\frac{\mathcal{N}_{c}(\tau, \mathcal{S})}{\delta}+(2+a)\tau\bigg)\\
    + & (1+a)\bigg(\frac{1+3/a}{n}\bigg(\frac{d(T-\log(t_0))}{T-t_0}+C_d^2\bigg)\log\frac{1}{\delta}+(2+a)\tau\bigg) \\
    + & (1+a)^2 d\varepsilon^2 + (2a+a^2)E_2\,,
\end{split}
\end{equation*}

with probability at least $1- 2\delta- 2L\exp(-\Omega(m))$ over the randomness of initialization $ \bm{W}$.

Let $a = \varepsilon^2$ and $\tau = \frac{1}{n}$, then we have:
\begin{equation*}
    \frac{1}{T-t_0}\int_{t_0}^{T} \left \| \nabla \log p_t(\cdot) - \hat{\bm{s}}(\cdot,t)\right \|_{\ell^2(p_t)}^2 \mathrm{d}t \lesssim \frac{1}{n\varepsilon^2}\bigg(\frac{d(T-\log(t_0))}{T-t_0}+C_d^2\bigg)\log\frac{\mathcal{N}_{c}(\frac{1}{n}, \mathcal{S})}{\delta}+\frac{1}{n} + d\varepsilon^2\,,
\end{equation*}

with probability at least $1- 2\delta- 2L\exp(-\Omega(m))$ over the randomness of initialization $ \bm{W}$.
\end{proof}

\section{Discussion of Lipschitz property of score function}
\label{sec:discussion}

Here we provide an example to illustrate how the Lipschitz constant of the score function in a causal model is related to the model's nonlinear functions.

Here we give an example with $d=3$, the causality is $x^{(1)} \Rightarrow x^{(2)} \Rightarrow x^{(3)} $.

According to~\cref{eq:causal_score_function}, we have that:

\begin{equation*}
    s_1(\bm{x}) = -\frac{x^{(1)}}{\sigma_1^2} +\frac{\partial f_2(x^{(1)})}{\partial x^{(1)}}\frac{\epsilon_2}{\sigma_2^2}, \quad s_2(\bm{x}) = \frac{f_2(x^{(1)}) - x^{(2)}}{\sigma_2^2} +\frac{\partial f_3(x^{(2)})}{\partial x^{(2)}}\frac{\epsilon_3}{\sigma_3^2}, \quad s_3(\bm{x}) = \frac{f_3(x^{(2)}) - x^{(3)}}{\sigma_3^2}\,.
\end{equation*}

Then we can derive that:

\begin{equation*}
    \frac{\partial s_1(\bm{x})}{\partial x^{(2)}} = \frac{\partial s_1(\bm{x})}{\partial x^{(3)}} = \frac{\partial s_2(\bm{x})}{\partial x^{(3)}} = \frac{\partial s_3(\bm{x})}{\partial x^{(1)}} = \frac{\partial s_3(\bm{x})}{\partial x^{(2)}} = 0\,,
\end{equation*}

\begin{equation*}
\frac{\partial s_1(\bm{x})}{\partial x^{(1)}} = -\frac{1}{\sigma_1^2} + \frac{\partial^2 f_2(x^{(1)})}{\partial x^{(1)2}}\frac{\epsilon_2}{\sigma_2^2}\,,
\end{equation*}

\begin{equation*}
\frac{\partial s_2(\bm{x})}{\partial x^{(2)}} = -\frac{1}{\sigma_2^2} + \frac{\partial^2 f_3(x^{(2)})}{\partial x^{(2)2}}\frac{\epsilon_3}{\sigma_3^2}\,,
\end{equation*}

\begin{equation*}
\frac{\partial s_2(\bm{x})}{\partial x^{(1)}} = \frac{\partial^2 f_3(x^{(2)})}{\partial x^{(2)}\partial x^{(1)}}\frac{\epsilon_3}{\sigma_3^2}\,,
\end{equation*}

\begin{equation*}
\frac{\partial s_3(\bm{x})}{\partial x^{(3)}} = -\frac{1}{\sigma_3^2}\,.
\end{equation*}

We denote $\bm{J}$ as the Jacobian of the score function. Then we can derive:

\begin{equation*}
\begin{split}
    \left \| \bm{J} \right \|_{\ell_\infty} & = \max\bigg(\left | -\frac{1}{\sigma_1^2} + \frac{\partial^2 f_2(x^{(1)})}{\partial x^{(1)2}}\frac{\epsilon_2}{\sigma_2^2} \right |  ,\left | -\frac{1}{\sigma_2^2} + \frac{\partial^2 f_3(x^{(2)})}{\partial x^{(2)2}}\frac{\epsilon_3}{\sigma_3^2} \right | + \left | \frac{\partial^2 f_3(x^{(2)})}{\partial x^{(2)}\partial x^{(1)}}\frac{\epsilon_3}{\sigma_3^2} \right | ,\frac{1}{\sigma_3^2}\bigg)\\
    & \leq \max\bigg(\frac{1}{\sigma_1^2} + \left |\sup \frac{\partial^2 f_2(x^{(1)})}{\partial x^{(1)2}}\frac{\epsilon_2}{\sigma_2^2} \right |, \frac{1}{\sigma_2^2} + \left | \sup\frac{\partial^2 f_3(x^{(2)})}{\partial x^{(2)2}}\frac{\epsilon_3}{\sigma_3^2} \right | + \left | \sup\frac{\partial^2 f_3(x^{(2)})}{\partial x^{(2)}\partial x^{(1)}}\frac{\epsilon_3}{\sigma_3^2} \right | ,\frac{1}{\sigma_3^2}\bigg)\,.
\end{split}
\end{equation*}

Then we have that for any $L$ satisfy:

\begin{equation*}
    L\geq \max\bigg(\frac{1}{\sigma_1^2} + \left |\sup \frac{\partial^2 f_2(x^{(1)})}{\partial x^{(1)2}}\frac{\epsilon_2}{\sigma_2^2} \right |, \frac{1}{\sigma_2^2} + \left | \sup\frac{\partial^2 f_3(x^{(2)})}{\partial x^{(2)2}}\frac{\epsilon_3}{\sigma_3^2} \right | + \left | \sup\frac{\partial^2 f_3(x^{(2)})}{\partial x^{(2)}\partial x^{(1)}}\frac{\epsilon_3}{\sigma_3^2} \right | ,\frac{1}{\sigma_3^2}\bigg)\,,
\end{equation*}

then the $L$ is one of the Lipschitz constants of the score function.

According to the previous analysis, we can obtain the Lipschitz property of the score function by imposing some assumptions on the nonlinear function and noise of the model. For causal models with more complex DAG, the relationship between Lipshcitz and the model will be more complicated, but as long as the second derivatives of nonlinear functions are bounded and the variance of the noise is non-zero, the score function has Lipschitz property, and the value of the Lipschitz constant depends on the nonlinear function and noise of the model.

\section{Broader Impacts}
\label{sec:broader_impacts}

This is a theoretical work that provides theoretical analysis for causal inference based on score matching. As such, we do not expect our work to have negative societal bias, as we do not focus on obtaining state-of-the-art results in a particular task. On the contrary, our work can have various benefits for the community: 

\begin{itemize}
    \item Causal inference is crucial in fields such as medicine, social sciences, and economics for understanding the essence of phenomena and formulating effective intervention measures. The outcomes of this work not only provide researchers in these fields with more reliable and interpretable theoretical insights into causal inference, driving scientific advancements and societal development, but also the score matching-based causal inference methods can help uncover hidden causal effects and mechanisms, providing a scientific foundation for decision-making in areas such as social equity, educational policies, and medical interventions.
    \item The theoretical framework and methods developed in this work can inspire and inform other causal inference approaches, fostering interdisciplinary research and collaboration, and expanding the application scope of causal inference in different domains.
\end{itemize}

\end{document}